\ifdefined\colt
\documentclass[12pt,final]{colt2019} 
\title[An Information-Theoretic Approach to Minimax Regret in Partial Monitoring]
{An Information-Theoretic Approach to Minimax Regret in Partial Monitoring}
\else
\documentclass[12pt]{colt2019-clean} 
\title[An Information-Theoretic Approach to Minimax Regret in Partial Monitoring]
{\normalfont\textsc{An Information-Theoretic Approach to Minimax Regret in Partial Monitoring}}
\fi

\usepackage{times}
\usepackage{amsmath}
\usepackage[skins]{tcolorbox}
\usepackage{amssymb}
\usepackage{pgfplots}
\usepackage{pgfplotstable}
\usepackage{dsfont}
\usepackage{tikz}
\usetikzlibrary{arrows.meta,calc,decorations.markings,math,arrows.meta}
\usepackage{natbib}
\usepackage{setspace}
\usepackage{mathrsfs}
\usepackage{booktabs}
\usepackage{bm}
\usepackage{mdframed}
\usepackage{colortbl}
\usepackage{wrapfig}
\usepackage{mathtools}
\usepackage{float}

\newenvironment{simplealg}{\begin{mdframed}\setlength{\parindent}{0cm}\setlength{\parskip}{0.2cm}\tt\vspace{-0.15cm}}{\end{mdframed}}
\newcommand{\algind}{\hspace{1cm}}

\definecolor{dkblue}{cmyk}{1,.54,.04,.19} 
\hypersetup{
      bookmarks=true,         
      unicode=false,          
      pdftoolbar=true,        
      pdfmenubar=true,        
      pdffitwindow=false,     
      pdfstartview={FitH},    
      pdftitle={Partial monitoring},    
      pdfauthor={Lattimore and Szepesvari},     
      pdfsubject={Bandits},   
      pdfcreator={pdflatex},   
      pdfproducer={Producer}, 
      pdfkeywords={bandits} {statistics} {machine learning}, 
      pdfnewwindow=true,      
      colorlinks=true,        
      linkcolor=dkblue,       
      citecolor=dkblue,       
      filecolor=dkblue,       
      urlcolor=dkblue,        
}

\usepackage[capitalise]{cleveref}

\newcommand\numberthis{\addtocounter{equation}{1}\tag{\theequation}}
\newcommand{\argmax}{\operatornamewithlimits{arg\,max}}
\newcommand{\argmin}{\operatornamewithlimits{arg\,min}}

\newcommand{\R}{\mathbb{R}}
\newcommand{\E}{\mathbb{E}}

\newcommand{\interior}{\operatorname{int}}
\newcommand{\KL}{\operatorname{D}}
\newcommand{\bKL}[2]{\KL\left(#1 \bigvert #2\right)}

\newcommand{\bbP}{\mathbb P}
\newcommand{\borel}{\mathfrak{B}}

\newcommand{\cF}{\mathcal F}

\newcommand{\cN}{\mathcal N}
\newcommand{\cD}{\mathcal D}
\newcommand{\cA}{\mathcal A}
\newcommand{\cC}{\mathcal C}
\newcommand{\cT}{\mathcal T}
\newcommand{\cP}{\mathcal P}
\newcommand{\cK}{\mathcal K}
\newcommand{\cL}{\mathcal L}
\newcommand{\cX}{\mathcal X}
\newcommand{\ip}[1]{\langle #1 \rangle}

\newcommand{\sP}{\mathscr P}

\newcommand{\cH}{\mathcal H}

\newcommand{\sQ}{\mathscr Q}

\newcommand{\ones}{\bm{1}}

\newcommand{\norm}[1]{\left\Vert #1\right\Vert}
\newcommand{\dom}{\operatorname{dom}}
\newcommand{\diam}{\operatorname{diam}}
\newcommand{\image}{\operatorname{im}}

\newcommand{\tr}{\operatorname{tr}}

\newcommand{\PiD}{\Pi_{\textrm{D}}}
\newcommand{\PiDM}{\Pi_{\textrm{DM}}}
\newcommand{\PiM}{\Pi_{\textrm{M}}}

\definecolor{c1}{HTML}{316AC6}
\definecolor{c2}{HTML}{A88C1A}
\definecolor{c3}{HTML}{4F8436}
\definecolor{c4}{HTML}{8E2F78}
\definecolor{c5}{HTML}{A02026}
\definecolor{c6}{HTML}{1B9996}
\definecolor{c7}{HTML}{68720A}
\definecolor{c8}{HTML}{9E9E9E}
\definecolor{c9}{HTML}{9B0F0F}

\renewcommand{\mid}{\,|\,}

\newcommand{\bigvert}{\,\middle|\hspace{-1.8pt}\middle|\,}
\newcommand{\littlevert}{\,|\hspace{-1.8pt}|\,}

\newcommand{\BReg}{\mathfrak{BR}}
\newcommand{\Reg}{\mathfrak{R}}
\newcommand{\reg}{\mathfrak{r}}

\newcommand{\sun}{\includegraphics[height=1em]{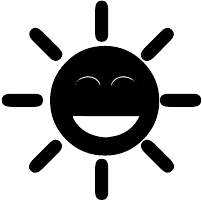}}
\newcommand{\rain}{\includegraphics[height=1em]{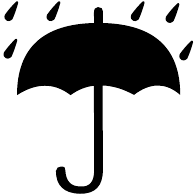}}
\newcommand{\snow}{\includegraphics[height=1em]{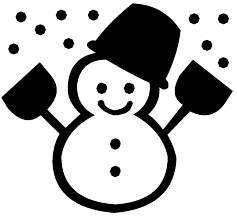}}
\newcommand{\isun}{\includegraphics[height=0.75em]{img/sun}}
\newcommand{\irain}{\includegraphics[height=0.75em]{img/rain}}
\newcommand{\isnow}{\includegraphics[height=0.75em]{img/snow}}

\let\tmp\epsilon
\let\epsilon\varepsilon
\let\varepsilon\tmp

\coltauthor{%
 \Name{Tor Lattimore} \Email{lattimore@google.com}\\
 \addr DeepMind, London
 \AND
 \Name{Csaba Szepesv\'ari} \Email{szepi@google.com}\\
 \addr DeepMind, London
}

\begin{document}

\maketitle

\begin{abstract}%
We prove a new minimax theorem connecting the worst-case Bayesian regret and minimax regret under finite-action partial monitoring with no assumptions on the space of signals or decisions of the adversary.
We then generalise the information-theoretic tools of \cite{RV16} for proving Bayesian regret bounds and combine them with the minimax theorem to derive minimax regret bounds for various
partial monitoring settings. The highlight is a clean analysis of `easy' and `hard' finite partial monitoring, with new regret bounds 
that are independent of arbitrarily large game-dependent constants and eliminate the logarithmic dependence on the horizon for easy games that appeared in earlier work.
The power of the generalised machinery is further demonstrated by proving that the minimax regret for $k$-armed adversarial bandits
is at most $\sqrt{2kn}$, improving on existing results by a factor of $2$. Finally, we provide a simple analysis of the cops and robbers game, also improving best known constants.
\end{abstract}


\section{Introduction}

Partial monitoring is a generalisation of the multi-armed bandit framework with an interestingly richer structure. In this paper we are concerned with the finite-action version.
Let $k$ be the number of actions. 
A finite-action partial monitoring game is described by two functions, the loss function $\cL : [k] \times \cX \to [0,1]$
and a signal function $\Phi : [k] \times \cX \to \Sigma$, where $[k] = \{1,2,\ldots,k\}$ and $\cX$ and $\Sigma$ are topological spaces.
At the start of the game the adversary secretly chooses a sequence 
of outcomes $(x_t)_{t=1}^n$ with $x_t \in \cX$, where $n$ is the horizon. 
The learner knows $\cL$, $\Phi$ and $n$ and sequentially chooses actions $(A_t)_{t=1}^n$ from $[k]$. 
In round $t$, after the learner chooses $A_t$ they suffer a loss of $\cL(A_t, x_t)$ and observe 
only $\Phi(A_t, x_t)$ as a way of indirectly learning about the loss.
A policy $\pi$ is a function mapping action/signal sequences to probability distributions over actions (the learner is allowed to randomise) and the regret of policy $\pi$ in environment $x = (x_t)_{t=1}^n$ is
\begin{align*}
\Reg_n(\pi, x) = \max_{a \in [k]}\E\left[\sum_{t=1}^n \cL(A_t, x_t) - \cL(a, x_t)\right]\,,
\end{align*}
where the expectation is taken with respect to the randomness in the learner's choices which follow $\pi$.
The minimax regret of a partial monitoring game is
\begin{align*}
\Reg_n^* = \inf_{\pi \in \Pi} \sup_{x \in \cX^n} \Reg_n(\pi, x)\,,
\end{align*}
where $\Pi$ is the space of all policies. 
Our objective is to understand how the minimax regret depends on the horizon $n$ and the structure of $\cL$ and $\Phi$.
Note, this is the oblivious setting because the adversary chooses all the losses at the start of the game.
Some classical examples of partial monitoring games are given in \cref{table:envs} and \cref{fig:example} in the appendix.

\begin{table}[h!]
\renewcommand{\arraystretch}{1.5}
\begin{tabular}{@{}p{4cm}p{1.2cm}p{1.2cm}p{4.9cm}p{2cm}}
\toprule 
\textbf{Setting} & $\bm \cX$ & $\bm \Sigma$ & $\bm{\Phi(a, x)}$ & $\bm{\cL(a, x)}$ \\
\midrule
Full information & $[0,1]^k$ & $[0,1]^k$ & $x$ & $x_a$ \\ 
Bandit & $[0,1]^k$ & $[0,1]$ & $x_a$ & $x_a$ \\
Cops and robbers & $[0,1]^k$ & $[0,1]^{k-1}$ & $x_1,\ldots,x_{a-1},x_{a+1},\ldots,x_k$ & $x_a$ \\
Finite partial monitoring & $[d]$ & arbitrary & arbitrary & arbitrary \\
\bottomrule
\end{tabular}
\caption{Example environment classes. 
In the last row, $d$ is a natural number.}\label{table:envs}
\end{table}

\paragraph{Bayesian viewpoint} 
Although our primary objective is to shed light on the minimax adversarial regret, we establish our results by first proving
uniform bounds on the Bayesian regret that hold for any prior. Then a new minimax theorem demonstrates the existence of an
algorithm with the same minimax regret. While these methods are not constructive, we demonstrate that they lead to elegant analysis of various partial monitoring problems, and better control 
of the constants in the bounds.

Let $\sQ$ be a space of probability measures on $\cX^n$ with the Borel $\sigma$-algebra.
The Bayesian regret of a policy $\pi$ with respect to prior $\nu \in \sQ$ is
\begin{align*}
\BReg_n(\pi, \nu) = \int_{\cX^n} \Reg_n(\pi, x) d\nu(x)\,.
\end{align*}
The minimax Bayesian optimal regret is
\begin{align*}
\BReg_n^*(\sQ) = \sup_{\nu \in \sQ} \inf_{\pi \in \PiM} \BReg_n(\pi, \nu)\,,
\end{align*}
where $\PiM$ is a space of policies so that $x \mapsto \Reg_n(\pi, x)$ is measurable, which we define formally in \cref{sec:notation}. 
When $\sQ$ is clear from the context, we write $\BReg_n^*$ in place of $\BReg_n^*(\sQ)$.

\paragraph{Contributions}

Our first contribution is to generalise the machinery developed by \cite{RV16,RV17} and \cite{BDKP15}.
In particular, we prove a minimax theorem for finite-action partial monitoring games with no restriction on either the loss or the feedback function. 
The theorem establishes that the Bayesian optimal regret and minimax regret are equal: $\BReg_n^* = \Reg_n^*$.
Next, the information-theoretic machinery of \cite{RV17} is generalised by replacing the mutual information with an expected Bregman divergence.
The power of the generalisation is demonstrated by showing that $\Reg^*_n \leq \sqrt{2kn}$ for $k$-armed adversarial bandits, which
improves on the best known bounds by a factor of $2$.
The rest of the paper is focussed on applying these ideas to finite partial monitoring games.
The results enormously simplify existing analysis by sidestepping the complex localisation arguments.
At the same time, our bounds for the class of `easy non-degenerate' games do not depend on arbitrarily large game-dependent constants, which
was true of all prior analysis. Finally, for a special class of bandits with graph feedback called cops and robbers, we show that $\Reg^*_n \leq \sqrt{2n \log(k)}$, 
improving on prior work by a factor of $5/\sqrt{2}$.

\section{Related work}

Since partial monitoring is so generic, the related literature is vast, with most work focussing on the full information 
setting (see \cite{Ces06}) or the bandit setting (\cite{BC12,LS19bandit-book}).
The information-theoretic machinery that we build on was introduced by \cite{RV16,RV17} in the context of 
minimizing the Bayesian regret for stationary stochastic bandits (with varying structural assumptions).
\cite{BDKP15} noticed the results also applied to the `adversarial' Bayesian setting and applied minimax theory to prove worst-case bounds
for convex bandits. Minimax theory has also been used to transfer Bayesian regret bounds to adversarial bounds.
For example, \cite{AAB09} explores this in the context of online convex optimisation in the full-information setting and \cite{GPS16} 
for prediction with expert advice.
The finite version of partial monitoring was introduced by \cite{Rus99}, who developed Hannan consistent algorithms.
The main challenge since then has been characterizing the dependence of the regret on the horizon in terms of the structure of the loss and signal functions.
It is now known that all games can be classified into one of exactly four types. Trivial and hopeless, for which $\Reg^*_n = 0$ and $\Reg^*_n = \Omega(n)$ respectively.
Between these extremes there are `easy' games where $\Reg^*_n = \Theta(n^{1/2})$ and `hard' games for which $\Reg^*_n = \Theta(n^{2/3})$.
The classification result is proven by piecing together upper and lower bounds from various papers \citep{CBLuSt06,FR12,ABPS13,BFPRS14,LS18pm}.
A caveat of the classification theorem is that the focus is entirely on the dependence of the minimax regret on the horizon.
The leading constant is game-dependent and poorly understood. Existing bounds for easy games depend on a constant that can be arbitrarily large, even for fixed $d$ and $k$.
One of the contributions of this paper is to resolve this issue.
Another disadvantage of the current partial monitoring literature, especially in the adversarial setting, is that the algorithms and analysis tend to be rather complicated.
Although our results only prove the existence of an algorithm witnessing a claimed minimax bound, the Bayesian algorithm and analysis are intuitive and natural.
There is also a literature on stochastic partial monitoring, with early analysis by \cite{BPS11}. 
A quite practical algorithm was proposed by \cite{VBK14}. The asymptotics have also been worked out \citep{KHN15}. 
Although a frequentist regret bound in a stochastic setting normally
implies a Bayesian regret bound, in our Bayesian setup the environments are not stationary, while all the algorithms for the stochastic case rely heavily that the distribution of the adversary is stationary. 
Generalising these algorithms to the non-stationary case does not seem straightforward.
Finally, we should mention there is an alternative definition of the regret that is less harsh on the learner. 
For trivial, easy and hard games it is the same, but for hopeless games the regret captures the hopelessness of the task and measures the performance of the learner relative
to an achievable objective. We do not consider this definition here. Readers interested in this variation can consult the papers by \cite{Rus99,MS03,Per11} and \cite{MPS14}.


\section{Notation and conventions}\label{sec:notation}

The maximum/supremum of the empty set is negative infinity.
The standard basis vectors in $\R^d$ are $e_1,\ldots,e_d$. The column vector of all ones is $\ones = (1,1,\ldots,1)^\top$. 
The standard inner product is $\ip{\cdot, \cdot}$. The $i$th coordinate of vector $x \in \R^d$ is $x_i$.
The $(d-1)$-dimensional probability simplex is $\Delta^{d-1} = \{x \in [0,1]^d : \norm{x}_1 = 1\}$.
The interior of a topological space $Z$ is $\interior(Z)$ and its boundary is $\partial Z$.
The relative entropy between probability measures $\mu$ and $\nu$ over the same measurable space is 
$\bKL{\nu}{\mu} = \int \log(\frac{d\nu}{d\mu}) d\nu$ if $\nu \ll \mu$ and $\bKL{\nu}{\mu} = \infty$ otherwise, where $\log$ is the natural logarithm.
When $X$ is a random variable with $X \in [a,b]$ almost surely, then Pinsker's inequality combined with straightforward inequalities shows that
\begin{align}
\int X (d\mu - d\nu) \leq (b - a)\norm{\mu - \nu}_{\text{TV}} \leq (b - a) \sqrt{\frac{1}{2}\bKL{\mu}{\nu}}\,, 
\label{eq:pinsker}
\end{align}
where $\norm{\mu - \nu}_{\text{TV}}$ is the total variation distance.  When $\nu \ll \mu$, the squared Hellinger distance can be written as $h(\nu, \mu)^2 = \int (1 - \sqrt{d\nu/d\mu})^2 d\mu$.
Given a measure $\bbP$ and jointly distributed random elements $X$ and $Y$ we let $\bbP_X$ denote the law of $X$ and (unconventionally) we let $\bbP_{X | Y}$ be the conditional law of $X$ given $Y$, which satisfies
$\bbP_{X | Y}(\cdot) = \bbP(X \in \cdot \mid Y)$.
One can think of $\bbP_{X|Y}$ as a random probability measure over the range of $X$ that depends on $Y$. In none of our analysis do we rely on exotic spaces where such regular versions do not exist.
When $Y \in [k]$ is discrete we let $\bbP_{X|Y=i}$ denote $\bbP(X\in \cdot \mid Y=i)$ for $i\in [k]$. 
With this notation the mutual information between $X$ and $Y$ is $I(X ; Y) = \E[\KL(\bbP_{X | Y}\littlevert \bbP_X)]$.
The domain of a convex function $F : \R^d \to \R \cup \{\infty\}$ is $\dom(F) = \{x : F(x) < \infty\}$.  The Bregman divergence with respect to convex/differentiable $F$
is $D_F : \dom(F) \times \dom(F) \to [0,\infty]$. For $x, y \in \dom(F)$ this is defined by $D_F(x, y) = F(x) - F(y) - \nabla_{x-y} F(y)$, where $\nabla_v F(y)$ is the directional derivative
of $F$ in direction $v$ at $y$.
The relative entropy between categorical distributions $p, q \in \Delta^{k-1}$ is the Bregman divergence between $p$ and $q$ where
$F$ is the unnormalised negentropy: $F(p) = \sum_{i=1}^k ( p_i \log(p_i) - p_i)$ with domain $[0,\infty)^k$.
The diameter of a convex set $\cK$ with respect to $F$ is $\diam_F(\cK) = \sup_{x, y \in \cK} F(x) - F(y)$.

\paragraph{Probability spaces, policies and environments} 
The Borel $\sigma$-algebra on topological space $Z$ is $\borel(Z)$. 
Recall that $\cX$ and $\Sigma$ are assumed to carry a topology, which we will use for ensuring measurability of the regret. 
More about the choices of these topologies later.
We assume the signal function $\Phi(a, \cdot)$ is $\borel(\cX)/\borel(\Sigma)$-measurable and
the loss function $\cL(a, \cdot)$ is $\borel(\cX)$-measurable.
A policy is a function $\pi : \cup_{t=1}^n ([k] \times \Sigma)^{t-1} \to \Delta^{k-1}$ and the space of all policies is $\Pi$.
A policy is measurable if $h_{t-1} \mapsto \pi(h_{t-1})$ is $\borel(([k] \times \Sigma)^t)$-measurable for all $h_{t-1} = a_1,\sigma_1,\ldots,a_{t-1},\sigma_{t-1}$, which coincides with 
the usual definition of a probability kernel. The space of all measurable policies is $\PiM$. In general $\PiM$ is a strict subset of $\Pi$.
For most of the paper we work in the Bayesian framework where there is a prior probability measure $\nu$ on $(\cX^n, \borel(\cX^n))$.
Given a prior $\nu$ and a measurable policy $\pi \in \PiM$, random elements $X \in \cX^n$ and $A \in [k]^n$ are defined on common probability space $(\Omega, \cF, \bbP)$. 
We let $\Phi_t(a) = \Phi(a, X_t)$ and $\cL_t(a) = \cL(a, X_t)$.
Expectations $\E$ are with respect to $\bbP$. 
For $t\in [n+1]$ we let $\cF_t = \sigma(A_1,\Phi(A_1,X_1),\ldots,A_{t-1},\Phi(A_{t-1}, X_{t-1})) \subseteq \cF$, 
$\E_t[\cdot] = \E[\cdot \mid \cF_t]$ and $\bbP_t(\cdot) = \bbP(\cdot \mid \cF_t)$. 
Note that $\cF_1 = \{\emptyset,\Omega\}$ is the trivial $\sigma$ algebra.
The $\sigma$-algebra $\cF$ and the measure $\bbP$ are such that 
\begin{enumerate}
\item The law of the adversaries choices satisfies $\bbP(X \in \cdot\,) = \nu(\cdot)$.
\item For any $t\in [n]$, the law of the actions almost surely satisfies 
\begin{align}
\bbP_t(A_t \in \cdot\,) 
&= \bbP_t(A_t \in \cdot \mid X) = \pi(A_1,\Phi_1(A_1),\ldots,A_{t-1},\Phi_{t-1}(A_{t-1}))(\cdot)\,.
\label{eq:keycond}
\end{align}
\end{enumerate}
The existence of a probability space satisfying these properties is guaranteed by Ionescu-Tulcea \citep[Theorem 6.17]{Kal06}.
The last condition captures the important assumption that, conditioned on the observed history, $A_t$ is sampled independently from $X$.
In particular, it implies that $X_t$ and $A_t$ are independent under $\bbP_t$.
The optimal action is $A^* = \argmin_{a \in [k]} \sum_{t=1}^n \cL_t(a)$.
It is not hard to see that the Bayesian regret is well defined and satisfies 
\begin{align*}
\BReg_n(\pi,\nu) = \E\left[\sum_{t=1}^n \cL_t(A_t) - \cL_t(A^*)\right] = \E\left[\sum_{t=1}^n \Delta_t\right]\,,
\end{align*}
where $\Delta_t = \cL_t(A_t) - \cL_t(A^*)$.
To minimise clutter, when the policy $\pi$ and prior $\nu$ are clear from the context, we abbreviate $\BReg_n(\pi,\nu)$ to $\BReg_n$.
We let $P_{ta} = \bbP_t(A_t = a)$, which means that $P_t \in \Delta^{k-1}$ is a probability vector. 

\section{Minimax theorem}

Our first main result is a theorem that connects the minimax regret to the worst-case Bayesian regret over all finitely supported priors.
The regret $\Reg_n(\pi, x)$ is well defined for any $x$ and any policy $\pi \in \Pi$, but the Bayesian regret depends on measurability of $x \mapsto \Reg_n(\pi, x)$.
If $\nu$ is supported on a finite set $x_1,\dots,x_m\in \cX^n$, however, we can write
\begin{align*}
\BReg_n(\pi,\nu) = \sum_{i=1}^m \nu( \{ x_i \} ) \Reg_n(\pi,x_i)\,,
\end{align*}
which does not rely on measurability.
By considering finitely supported priors we free ourselves from any concern that $x \mapsto \Reg_n(\pi, x)$ might not be measurable.
This also means that if $\Sigma$ (or $\cX$) came with some topologies, we simply replace them with the discrete topology (which makes all maps continuous and measurable,
implying $\Pi = \PiM$).

\begin{theorem}\label{thm:swap}
Let $\sQ$ be the space of all finitely supported probability measures on $\cX^n$. Then
\begin{align*}
 \inf_{\pi \in \Pi} \sup_{x \in \cX^n} \Reg_n(\pi, x)
 = 
 \sup_{\nu\in \sQ } \min_{\pi \in \Pi} \BReg_n(\pi, \nu)\,.
\end{align*}
\end{theorem}
An equivalent statement of this theorem is that if $\cX$ and $\Sigma$ carry the discrete topology then $\Reg_n^* = \BReg^*_n(\sQ)$, which is the form we prove in \cref{sec:sion}. 
The strength of this result is that it depends on no assumptions except that the action set is finite.

Our proof borrows techniques from a related result by \cite{BDKP15}.
The main idea is to replace the policy space $\Pi$ with a simpler space of `mixtures' over deterministic policies, which 
is related to Kuhn's celebrated result on the equivalence of behavioral and mixed strategies \citep{kuhn1953extensive}. 
We then establish that this space is compact and use Sion's theorem to exchange the minimum and maximum.
While we borrowed the ideas from \cite{BDKP15}, our proof relies heavily on the finiteness of the action space, which allowed us to avoid
any assumptions on $\Sigma$ and $\cX$, which also necessitated our choice of $\sQ$. Neither of the two results imply each other.

\cref{thm:swap} is a minimax theorem for a special kind of two-player multistage zero-sum deterministic partial information game.
Minimax theorems for this case are nontrivial because of challenges related to measurability and the use of Sion's theorem. 
Although there is a rich and sophisticated literature on this topic, we are not aware of any result implying our theorem. 
Tools include the approach we took using the weak topology \citep{Be92}, or the so-called weak-strong topology \citep{LedoVF00} and
reduction to completely observable games and then using dynamic programming \citep{GhMcDSi04}.
An interesting challenge is to extend our result to compact action spaces. One may hope to generalise the proof by \cite{BDKP15}, 
but some important details are missing (for example, the measurable space on which the priors live is undefined, the measurability of the 
regret is unclear as is the compactness of distributions induced by \emph{measurable} policies). 
We believe that the approach of \citet{GhMcDSi04} can complete this result.

\section{The regret information tradeoff}

Unless otherwise mentioned, all expectations $\E$ are with respect to the probability measure over interactions between a fixed policy $\pi \in \PiM$ and an environment sampled
from a prior $\nu$ on $(\cX^n, \borel(\cX^n))$. 
Before our generalisation we present a restatement of the
core theorem in the analysis by \cite{RV16}. Let $I_t(X ; Y)$ be the mutual information between $X$ and $Y$ under $\bbP_t$. 
Although the proof is identical, the setup here is different because the prior $\nu$ is arbitrary.

\begin{theorem}[\cite{RV16}]\label{thm:hammer} 
Suppose there exists a constant $\beta \geq 0$ such that $\E_t[\Delta_t] \leq \sqrt{\beta I_t(A^* ; \Phi_t(A_t), A_t)}$
almost surely for all $t$. Then $\BReg_n \leq \sqrt{n \beta \log(k)}$.
\end{theorem}

This elegant result provides a bound on the regret in terms of the information gain about the optimal arm.
Our generalisation replaces the information gain with an expected Bregman divergence.

\begin{theorem}\label{thm:general}
Let $(M_t)_{t=1}^{n+1}$ be an $\R^d$-valued martingale adapted to $(\cF_t)_{t=1}^{n+1}$ and $M_t \in \cD \subset \R^d$ almost surely for all $t$.
Then let $F$ be a convex function with $\diam_F(\cD) < \infty$.
Suppose there exist constants $\alpha, \beta \geq 0$ such that
$\E_t[\Delta_t] \leq \alpha + \sqrt{\beta \E_t[D_F(M_{t+1}, M_t)]}$
almost surely for all $t$. Then $\BReg_n \leq \alpha n + \sqrt{n \beta \diam_F(\cD)}$.
\end{theorem}

\begin{proof}
We calculate 
\begin{align}
\E_t[D_F(M_{t+1}, M_t)]
&= \E_t\left[\liminf_{h \to 0+} \left(F(M_{t+1}) - F(M_t) - \frac{F(M_t + h(M_{t+1} - M_t)) - F(M_t)}{h} \right)\right] \nonumber \\
&\leq \liminf_{h \to 0+} \left(\E_t\left[F(M_{t+1}) - F(M_t) - \frac{F((1 - h)M_t + hM_{t+1}) - F(M_t)}{h}\right]\right) \nonumber \\
&= \E_t\left[F(M_{t+1})\right] - F(M_t) + \liminf_{h \to 0+} \frac{F(M_t) - \E_t[F((1 - h) M_t + h M_{t+1})]}{h} \nonumber \\
&\leq \E_t\left[F(M_{t+1})\right] - F(M_t) + \liminf_{h \to 0+} \frac{F(M_t) - F(\E_t[(1 - h) M_t + h M_{t+1}])}{h} \nonumber \\
&= \E_t\left[F(M_{t+1})\right] - F(M_t)\,, \label{eq:expect-breg}
\end{align}
where the first inequality follows from Fatou's lemma and the second from convexity of $F$. The last equality is because $\E_t[M_{t+1}] = M_t$.
Hence
\begin{align*}
\BReg_n
&= \E\left[\sum_{t=1}^n \Delta_t\right] 
\leq \alpha n + \E\left[\sum_{t=1}^n \sqrt{\beta \E_t[D_F(M_{t+1}, M_t)]}\right] \\
&\leq \alpha n + \sqrt{\beta n \E\left[\sum_{t=1}^n \E_t[D_F(M_{t+1}, M_t)]\right]} 
\leq \alpha n + \sqrt{\beta n  \diam_F(\cD)}\,,
\end{align*}
where the first inequality follows from the assumption in the theorem, the second by Cauchy-Schwarz,
while the third follows by \cref{eq:expect-breg}, telescoping and the definition of the diameter.
\end{proof}

A natural choice for $M_t$ is the posterior distribution of the optimal action.
Let $P_{ta}^* = \bbP_t(A^* = a)$, which is the posterior probability that $A^* = a$ based on the information available at the start of round $t$.
By the tower rule, we have $\E_t[P_{t+1}^*] = P_t^*$ so that $(P_t^*)_{t=1}^{n+1}$ is a martingale adapted to $(\cF_t)_{t=1}^{n+1}$.

\begin{corollary}\label{cor:general}
Let $F : \R^k \to \R$ be a convex function with $\diam_F(\Delta^{k-1}) < \infty$.
Suppose there exist constants $\alpha, \beta \geq 0$ such that $\E_t[\Delta_t] \leq \alpha + \sqrt{\beta \E_t[D_F(P^*_{t+1}, P^*_t)]}$
almost surely for all $t$. Then $\BReg_n \leq \alpha n + \sqrt{n \beta \diam_F(\Delta^{k-1})}$.
\end{corollary}

\begin{remark}
That \cref{thm:general} generalises \cref{thm:hammer} follows by choosing $F$ as the unnormalised negentropy for which 
$\diam_F(\Delta^{k-1}) \leq \log(k)$ and 
$\E_t[D_F(P_{t+1}^*,P_t^*)] = I_t(A^* ; \Phi_t(A_t),A_t)$.
The assumption that $M_t \in \R^d$ can be relaxed. The result continues to hold when $M_t$ takes values in a bounded subset of a Banach space, where the martingale is defined using the 
Bochner integral. The Bregman divergence generalises naturally via the Gateoux derivative.
\end{remark}

\section{Finite-armed bandits}\label{sec:bandits}

In the bandit setting the learner observes the loss of the action they play, which is modelled by choosing $\Sigma = [0,1]$, $\cX = [0,1]^k$ and $\Phi(a, x) = \cL(a, x) = x_a$.
The best known bound is by \cite{BC12}, who prove that online mirror descent with an appropriate potential satisfies $\Reg_n^* \leq \sqrt{8kn}$.
Using the same potential in combination with \cref{thm:general} allows us to improve this result to $\Reg_n^* \leq \sqrt{2kn}$.

\begin{theorem}\label{thm:finite}
The minimax regret for $k$-armed adversarial bandits satisfies $\Reg^*_n \leq \sqrt{2kn}$.
\end{theorem}

\begin{proof}
Let $F(p) = -2 \sum_{a=1}^k \sqrt{p_a}$, which has domain $[0,\infty)^k$ and $\diam_F(\Delta^{k-1}) \leq 2\sqrt{k}$.
Combine Corollary~\ref{cor:general} and \cref{thm:swap}
and Lemma~\ref{lem:finite-lem} below for Thompson sampling, which is the policy that samples $A_t$ from $P_t = P_t^*$.
\end{proof}

\begin{lemma}\label{lem:finite-lem}
Let $F$ be as above and $P_t = P_t^*$. Then $\E_t[\Delta_t] \leq \sqrt{k^{1/2} \E_t[D_F(P_{t+1}^*, P_t^*)]}$ a.s..
\end{lemma}

\begin{remark}
Potentials other than the negentropy have been used in many applications in bandits and online convex optimisation.
The log barrier, for example, leads to first order bounds for $k$-armed bandits \citep{CL18}. 
Alternative potentials also appear in the context of adversarial linear bandits \citep{BCK12,BCL17} and
follow the perturbed leader \citep{ALST14}.
Investigating the extent to which these applications transfer to the Bayesian setting is an interesting direction for the future.
\end{remark}

\section{Finite partial monitoring games}

Recall from \cref{table:envs} that a finite partial monitoring game is characterised by functions $\cL : [k] \times [d] \to [0,1]$ and $\Phi : [k] \times [d] \to \Sigma$ where
$d$ is a natural number and $\Sigma$ is arbitrary. Finite partial monitoring enjoys a rich linear structure, which we now summarise.
A picture can help absorbing these concepts, and is provided with an example at the beginning of \cref{app:figures}.
For $a \in [k]$, let $\ell_a \in [0,1]^d$ be the vector with $\ell_{ax} = \cL(a,x)$.
Actions $a$ and $b$ are duplicates if $\ell_a = \ell_b$.
The cell associated with action $a$ is
$C_a = \{u \in \Delta^{d-1} : \ip{\ell_a, u} \leq \min_{b \neq a} \ip{\ell_b, u}\}$,
which is the subset of distributions $u \in \Delta^{d-1}$ where action $a$ minimises $\E_{x \sim u}[\cL(a, x)]$.
Note that $C_a \subset \R^d$ is a closed convex polytope and its dimension $\dim(C_a)$ is defined as the dimension of the affine space it generates.
An action $a$ is called Pareto optimal if $C_a$ has dimension $d-1$ and degenerate otherwise.
Of course $\cup_a C_a = \Delta^{d-1}$, but cells may have nonempty intersection. When $a$ and $b$ are not duplicates, the intersection $C_a \cap C_b$
is a (possibly empty) polytope of dimension at most $d - 2$.
A pair of Pareto optimal actions $a$ and $b$ are called neighbours if $C_a \cap C_b$ has dimension $d-2$. 
A game is called non-degenerate if there are no degenerate actions and no duplicate actions. 
So far none of the concepts have depended on the signal function. Local observability is a property of the signal and loss functions that allows the learner
to estimate loss differences between actions $a$ and $b$ by playing only those actions.
For neighbours $a$ and $b$
let $\cN_{ab} = \{c : C_c \subseteq C_a \cap C_b\}$, which contains $a$ and its duplicates, $b$ and its duplicates, and degenerate actions $c$ with $C_c = C_a \cap C_b$. 
A game is globally observable if for all pairs of neighbours there exists a function $f : [k] \times \Sigma \to \R$ such that
\begin{align}
\cL(a, x) - \cL(b, x) = \sum_{c=1}^k f(c, \Phi(c, x))\,.
\label{eq:est-def}
\end{align}
The game is locally observable if for all pairs of neighbours $a$ and $b$ the function $f$ can be chosen satisfying \cref{eq:est-def} and
additionally that $f(c, \Phi(c, x)) = 0$ for all $c \notin \cN_{ab}$.
In the standard analysis of partial monitoring the function $f$ is used to derive importance-weighted estimators of the loss differences.
In the following $f$ is used more directly. 
A quantity that appears naturally in the analysis is the supremum norm of the estimation functions $f$.
Given a globally observable game, we let $v \geq 0$ be the smallest value such that for all pairs of neighbours $a$ and $b$ there exists a function
satisfying \cref{eq:est-def} with $\norm{f}_\infty \leq v$. For locally observable games $v$ is defined in the same way, but with the additional restriction
that $f$ is supported on $\cN_{ab}$.
The neighbourhood of $a$ is $\cN_a = \{b : \dim(C_a \cap C_b) \geq d - 2\}$.
The neighbourhood graph over $[k]$ has edges $\{(a, b) : a, b\text{ are neighbours}\}$.
For non-degenerate games, the neighbourhood graph is connected.

The following theorem classifies all partial monitoring games into one of four categories.
All results were known previously except that previous upper bounds for locally observable games were $\Reg_n^* = O((n \log(n))^{1/2})$.

\begin{theorem}\label{thm:classification}
The minimax regret for finite partial monitoring game $G$ satisfies the following:
\begin{align*}
\Reg_n^* = 
\begin{cases}
0 & \text{if there are no neighbouring actions} \\
\Theta(n^{1/2}) & \text{if there are neighbouring actions and $G$ is locally observable} \\
\Theta(n^{2/3}) & \text{if $G$ is globally observable and not locally observable} \\
\Omega(n) & \text{otherwise}\,.
\end{cases}
\end{align*}
\end{theorem}

\paragraph{Summary of new results}
The main theorem is the following, which improves on previous bounds that all depend on arbitrarily large game-dependent constants, even when $k$ and $d$ are fixed.

\begin{theorem}\label{thm:pm}
For any locally observable non-degenerate game: $\Reg_n^* \leq k^{3/2} (d+1) \sqrt{8n \log(k)}$.
\end{theorem}

For degenerate locally observable games the bound differs only due to the increased norm of the estimation functions.
In particular, we have the following theorem, which improves on prior work in terms of constants and logarithmic factors \citep{LS18pm}.

\begin{theorem}\label{thm:local}
For any locally observable game: $\Reg_n^* \leq v k^{3/2} \sqrt{8n \log(k)}$, where $v$ is a bound on the supremum norm of the estimation functions.
\end{theorem}

The bound for globally observable games has the same order as the prior work, but with slightly improved constants \citep{CBLuSt06}.

\begin{theorem}\label{thm:pm:hard}
For any globally observable game: $\Reg_n^* \leq  3(nkv)^{2/3} (\log(k)/2)^{1/3}$, where $v$ is a bound on the supremum norm of the estimation functions.
\end{theorem}

Finally, for any locally/globally observable game, Lemma~\ref{lem:v-bound} in the appendix shows that the norm of the estimators is bounded by at most $v \leq d^{1/2} (1+k)^{d/2}$, which provides an
explicit upper bound that is independent of the loss and signal matrix. We believe the exponential dependence
on the dimension is unavoidable in general.

\section{Proof of Theorem~\ref{thm:pm}}

For this section we assume the game is non-degenerate and locally observable. 
Before the proof of \cref{thm:pm} we provide the algorithm, which seems to be novel among previous algorithms for partial monitoring.
Note that Thompson sampling can suffer linear regret in partial monitoring (\cref{app:ts-fail}).
Let $G_t = \argmin_{a \in [k]} \E_t[\cL_t(a)]$ be the greedy action that minimises the $1$-step Bayesian expected loss.
The idea is to define a directed tree with vertex set $[k]$ and root $G_t$ and where all paths lead to $G_t$.
A little notation is needed.
Define an undirected graph with vertices $V_t$ and edges $E_t$ by
$V_t = \{a \in [k] : \E_t[\cL_t(a)] = \E_t[\cL_t(G_t)]\}$ and 
$E_t = \{a, b \in V_t : a \text{ and } b \text{ are neighbours}\}$,
which is connected by Lemma~\ref{lem:structure2}. Note that $V_t = \{G_t\}$ when $G_t$ is unique, but this is not always the case.
For $a \in V_t$ let $\rho_t(a)$ be the length of the shortest path from $a$ to $G_t$ in $(V_t, E_t)$ with $\rho_t(G_t) = 0$ by definition. Let $\cP_t : [k] \to [k]$ be the `parent' function:
\begin{align*}
\cP_t(a) = \begin{cases}
\argmin_{b \in \cN_a} \E_t[\cL_t(b)] & \text{if } a \notin V_t \\
\argmin_{b \in \cN_a \cap V_t} \rho_t(b) & \text{otherwise}\,.
\end{cases}
\end{align*}
The following lemma is proven in \cref{app:structure}.

\begin{lemma}\label{lem:tree}
The directed graph over vertex set $[k]$ with an edge from $a$ to $b$ if $a \neq G_t$ and $b = \cP(a)$ is a directed tree with root $G_t$.
\end{lemma}

Let $\cA_t(a)$ be the set of ancestors of action $a$ in the tree defined in Lemma~\ref{lem:tree}. We adopt the convention that $a \in \cA_t(a)$.
By the previous lemma, $G_t \in \cA_t(a)$ for all $a$.
Let $\cD_t(a)$ be the set of descendants of $a$, which does not include $a$ (\cref{fig:tree}). The depth of an action $a$ in round $t$ is the distance between $a$ and the root $G_t$.
An action $a$ is called anomalous for $P \in \Delta^{k-1}$ in round $t$ if $P_a < \smash{\max_{b \in \cD_t(a)} P_b}$.
\cref{fig:flow} defines the `water transfer' operator $W_t : \Delta^{k-1} \to \Delta^{k-1}$ that corrects this deficiency by transferring mass towards the root of the tree defined in Lemma~\ref{lem:tree} while  
ensuring that (a) the loss suffered when playing the according to the transformed distribution does not increase and (b) the distribution is not changed too much.
The process is illustrated in \cref{fig:water} in \cref{app:lem:pm:P}, where you will also find the proof of the next lemma.

\begin{lemma}\label{lem:pm:P}
Let $P \in \Delta^{k-1}$ and $Q = W^k_t P = W_t \cdots W_t P$. Then:
\begin{enumerate}
\item $\sum_{a=1}^k Q_a \E_t[\cL_t(a)] \leq \sum_{a=1}^k P_a \E_t[\cL_t(a)]$.
\item $Q_a \leq Q_{\cP_t(a)}$ for all $a \in [k]$.
\item $Q_a \geq P_a / k$ for all $a \in [k]$.
\end{enumerate}
\end{lemma}

Our new algorithm samples $A_t$ from $P_t = W^k_t P^*_t$.
Because of the plumbing and randomisation, the new algorithm is called Mario sampling (\cref{alg:pm}). 
The proof of Theorem~\ref{thm:pm} follows immediately from Theorems~\ref{thm:swap} and \ref{thm:hammer} and the following lemma.

\begin{algorithm}[h!]
\begin{simplealg}
\textbf{input:} partial monitoring game $(\Sigma, \cL, \Phi)$ and prior $\nu$

\textbf{for} $t = 1,\ldots,n$

\algind compute $P_t^*$ and $P_t = W_t^k P_t^*$. Then sample $A_t \sim P_t$.
\end{simplealg}
\caption{Mario sampling}\label{alg:pm}
\end{algorithm}

\begin{lemma}\label{lem:pm:gamma}
For Mario sampling: $\E_t[\Delta_t] \leq (d+1) k^{3/2} \sqrt{8 I_t(A^* ; \Phi_t(A_t), A_t)}$\, a.s..
\end{lemma}

\begin{proof}
We assume an appropriate zero measure set is discarded so that we can omit the qualification `almost surely' for the rest of the proof.
By the first part of Lemma~\ref{lem:pm:P},
\begin{align}
\E_t[\Delta_t] 
&\leq \sum_{a=1}^k P_{ta}^*\left(\E_t[\cL_t(a)] - \E_t[\cL_t(a) \mid A^* = a]\right)\,. \label{eq:pm:delta}
\end{align}
For $b \neq G_t$ let $f_b, g_b : \Sigma \to \R$ be a pair of functions such that $\max\{\norm{f_b}_\infty, \norm{g_b}_\infty\} \leq d+1$
and $f_b(\Phi(b,x)) + g_b(\Phi(\cP_t(b),x)) = \cL(b, x) - \cL(\cP_t(b), x)$ for all $x \in [d]$.
The existence of such functions is guaranteed by Lemma~\ref{lem:sup-bound} and the fact that $\cP_t(b) \in \cN(b)$ and because we assumed the game is non-degenerate, locally observable.
The expected loss of $a$ can be decomposed in terms of the sum of differences to the root,
\begin{align}
\E_t[\cL_t(a)] 
&= \E_t\left[\cL_t(G_t) + \sum_{b \in \cA_t(a) \setminus \{G_t\}} \left(\cL_t(b) - \cL_t(\cP_t(b))\right)\right] \nonumber \\
&= \E_t\left[\cL_t(G_t) + \sum_{b \in \cA_t(a) \setminus \{G_t\}} f_b(\Phi_t(b)) + g_b(\Phi_t(\cP_t(b)))\right]\,.
\label{eq:pmloetlt}
\end{align}
In the same way,
\begin{align}
\E_t[\cL_t(a) \mid A^* = a] 
&= \E_t\left[\cL_t(G_t) + \sum_{b \in \cA_t(a) \setminus \{G_t\}} f_b(\Phi_t(b)) + g_b(\Phi_t(\cP_t(b)))\,\middle|\, A^* = a\right]\,.
\label{eq:pmloetltc}
\end{align}
Then, because $\cA_t(a)$ and $G_t$ are $\cF_t$-measurable, 
\begin{align*}
&\E_t[\Delta_t]
\leq \sum_{a=1}^k P_{ta}^* \left(\E_t[\cL_t(a)] - \E_t[\cL_t(a) \mid A^* = a]\right) 
	\tag{\cref{eq:pm:delta}}\\
& = 	\sum_{a=1}^k P_{ta}^* \Bigg[
\sum_{b \in \cA_t(a) \setminus \{G_t\}} 
	\left(\E_t[f_b(\Phi_t(b))] - \E_t[f_b(\Phi_t(b))\mid A^*=a]\right) \\
& \quad +
\sum_{b \in \cA_t(a) \setminus \{G_t\}} 	
	\left(\E_t[g_b(\Phi_t(\cP_t(b)))] - \E_t[g_b(\Phi_t(\cP_t(b)))\mid A^*=a]\right)\Bigg]
	\tag{\cref{eq:pmloetlt,eq:pmloetltc}} \\
&\leq (d+1) \sum_{a=1}^k P_{ta}^* \sum_{b \in \cA_t(a)} \sqrt{8 \bKL{\bbP_{t,\Phi_t(b) | A^* = a}}{\bbP_{t,\Phi_t(b)}}}
	 \tag{\cref{eq:pinsker}, $\KL\ge 0$}\\
&\leq k (d+1) \sqrt{8 \sum_{a=1}^k P_{ta}^* \sum_{b \in \cA_t(a)} P_{ta}^* \bKL{\bbP_{t,\Phi_t(b) | A^* = a}}{\bbP_{t,\Phi_t(b)}}} 
	 \tag{Cauchy-Schwarz}\\
&\leq k^{3/2} (d+1) \sqrt{8 \sum_{a=1}^k P_{ta}^* \sum_{b \in \cA_t(a)} P_{tb} \bKL{\bbP_{t,\Phi_t(b) | A^* = a}}{\bbP_{t,\Phi_t(b)}}} 
	 \tag{Lemma~\ref{lem:pm:P}, Part~3}\\
&\leq k^{3/2} (d+1) \sqrt{8 \sum_{a=1}^k P_{ta}^* \sum_{b=1}^k P_{tb} \bKL{\bbP_{t,\Phi_t(b)|A^* = a}}{\bbP_{t,\Phi_t(b)}}} 
	 \tag{$\KL\ge 0$} \\
&= k^{3/2} (d+1) \sqrt{8 I_t(A^* ; \Phi_t(A_t), A_t)}\,.
	 \tag{Lemma~\ref{lem:fundamental}}
\end{align*}
\end{proof}

\begin{remark}\label{rem:m}
In many games there exists a constant $m$ such that $|\cA_t(a)| \leq m$ almost surely for all $a$ and $t$.
In this case Part~3 of Lemma~\ref{lem:pm:P} improves to $P_{ta} \geq P_{ta}^* / m$ and the application of Cauchy-Schwarz in Lemma~\ref{lem:pm:gamma}
can be strengthened. This means the bound in \cref{thm:pm} becomes $m (d+1) \sqrt{8kn \log(k)}$. 
For the game illustrated in \cref{fig:tree}, $m = 5$ while $k = 7$, but more extreme examples are easily constructed.
\end{remark}

\begin{algorithm}[h!]
\begin{simplealg}
\textbf{input:} $P \in \Delta^{k-1}$ and tree determined by $\cP_t$ 

find action $a$ at the greatest depth such that $P_a < \max_{b \in \cD_t(a)} P_b$.

if no such action is found, let $W_tP = P$ and return.

for $\alpha \in [0,1]$ let $\cD_t(a ; \alpha) = \{b \in \cD_t(a) : P_b \geq \alpha\}$.

let $\alpha^*$ be the largest $\alpha \in \{P_b : b \in \cD_t(a)\}$ such that
\begin{align*}
p_\alpha = \frac{1}{1 + |\cD_t(a ; \alpha)|} \sum_{b \in \cD_t(a ; \alpha) \cup\{a\}} P_b > q_\alpha= \max\{P_b : b \in \cD_t(a) \setminus \cD_t(a ; \alpha)\}\,. 
\end{align*}

let $(W_t P)_b = p_{\alpha^*}$ if $b \in \cD_t(a ; \alpha^*)\cup\{a\}$ and $(W_tP)_b = P_b$ otherwise.
\end{simplealg}
\caption{The water transfer operator $W_t : \Delta^{k-1} \to \Delta^{k-1}$.}\label{fig:flow}
\end{algorithm}

\section{Discussion and future directions}\label{sec:discussion}

One of the main benefits of the information-theoretic approach is the simplicity and naturality of the arguments, which 
is particularly striking in partial monitoring.
Even for the $k$-armed bandit analysis there is no tuning of learning rates or careful bounding of dual norms.
In exchange, our results are existential, though we emphasise that the Bayesian setting is interesting in its own right.
We anticipate that \cref{thm:general} will have many other applications and there is clearly more to understand about this generalisation. 
Is it a coincidence that the same potential leads to minimax bounds using both online stochastic mirror descent and Thompson sampling?

\paragraph{Information-directed sampling}
Thompson sampling depends on the prior, but not the potential that appears in \cref{thm:general}.
\cite{RV14} noted that the information-theoretic analysis is tightest when the algorithm is chosen to minimize $\E_t[\Delta_t]^2 / \E_t[D_F(P_{t+1}^*, P_t^*)]$, where $F$ is the unnormalised negentropy.
Our generalisation provides a means of constructing new algorithms by changing the potential. 

\paragraph{Open problems}
An obvious next step is stress test the applicability of \cref{thm:general}. 
Bandits with graph feedback beyond cops and robbers might be a good place to start \citep{ACDK15}.
One may also ask whether in adversarial linear bandits
the results by \cite{BCL17} can be replicated or improved using \cref{thm:general}.
There are many open problems in partial monitoring, a few of which we now describe. We hope some readers will be inspired to work on them!

\paragraph{Adaptivity} 
There exist games where for `nice' adversaries the regret should be $O(n^{1/2})$ while for truly adversarial data
the regret is as large as $\Theta(n^{2/3})$. Designing algorithms that adapt to a broad range of adversaries is an interesting challenge.
Some work on this topic in the stochastic setting is by \cite{BZS12}.
A related question is understanding how to use the information-theoretic machinery to provide adaptive bounds.

\paragraph{Constants} 
Our results have eliminated arbitrarily large constants from the analysis of easy non-degenerate games. 
Still, we do not yet understand how the regret should depend on the structure of $\cL$ or $\Phi$ except in special cases. The result in Remark~\ref{rem:m} is a small step in this direction, but 
there is much to do. The best place to start is probably lower bounds. Currently generic lower bounds for finite partial monitoring focus on the dependence on the horizon.
One concrete question is whether or not the minimum supremum norm of the estimation functions that appears in \cref{thm:pm:hard} is a fundamental quantity.

\paragraph{Stochastic analysis of Mario sampling}
Theorem~\ref{thm:hammer} and Lemma~\ref{lem:pm:gamma} show that for any prior Mario sampling satisfies 
$\BReg_n \leq k^{3/2}(d+1)\sqrt{8n\log(k)}$.
In the stationary stochastic setting we expect that for a suitable prior it should be possible to prove a bound on the frequentist regret of this algorithm.
Perhaps the techniques developed by \cite{AG12} or \cite{KKM12} generalise to this setting.

\bibliography{all}

\appendix

\section{Proof of Theorem~\ref{thm:swap}}\label{sec:sion}

The proof depends on a little functional analysis. The important point is that the space of policies written as probability measures over deterministic policies is compact and
the Bayesian regret is linear and continuous as a function of the measures over policies and priors over environments. Then minimax theorems can be used to exchange the $\min$ and $\sup$.
Guaranteeing compactness and continuity and avoiding any kind of measurability issues requires careful choice of topologies.
 
For a topological space $Z$, let $\sP_r(Z)$ be the space of Radon probability measures when $Z$ is equipped with the Borel $\sigma$-algebra.
The weak* topology on $\sP_r(Z)$ is the coarsest topology such that $\mu \mapsto \int f d\mu$ is continuous for all bounded continuous functions $f : Z \to \R$. 

Recall that $\cX$ is the space of outcomes and $\Sigma$ is the space of feedbacks and these are arbitrary sets.
A deterministic policy can be represented as a function $\pi : \cup_{t=1}^n \Sigma^{t-1} \to [k]$.
By the choice of topology on $\Sigma$, these are all continuous, hence, measurable.
Let $\PiD$ be the space of all such policies, $\PiDM$ be the space of the measurable policies amongst these.
By Tychonoff's theorem, $\PiD$ is compact with the product topology, where $[k]$ has the discrete topology.
$\PiD$ is Hausdorff because the product of Hausdorff spaces is Hausdorff.
By Theorem 8.9.3 in the two volume book by \cite{Bog07}, the space $\sP_r(\PiD)$ is weak*-compact. 
Clearly $\sP_r(\PiD)$ is also convex.

Let $\sQ$ be the space of finitely supported probability measures on $\cX^n$, which is a convex subset of $\sP_r(\cX^n)$ where $\cX^n$ is taken to have the discrete topology.  Equip with $\sQ$ with the weak* topology.
If $f =f(\mu,\nu)$ with $f: \sP_r(\PiD) \times \sQ \to \R$ is linear and continuous in both $\mu$ and $\nu$ individually. Since $\sP_r(\PiD)$ is compact, 
by Sion's minimax theorem \citep{Sio58},\footnote{Sion's theorem is more general, it only assumes that $f$
is quasiconvex/quasiconcave in each argument and upper/lower semicontinuous respectively.}
\begin{align*}
\min_{\mu \in \sP_r(\PiD)} \sup_{\nu \in \sQ} f(\mu, \nu) = \sup_{\nu \in \sQ} \min_{\mu \in \sP_r(\PiD)} f(\mu, \nu)\,.
\end{align*}
We are going to choose $f$ to be the Bayesian regret and argue that $\Pi$ can be identified with $\sP_r(\PiD)$. 
First, we need to check some continuity conditions for the regret.
Since $\cX^n$ has the discrete topology the map $x \mapsto \Reg_n(\pi, x)$ is continuous for fixed $\pi$.
Now we check that $\pi \mapsto \Reg_n(\pi, x)$, $\pi \in \PiD$, is continuous for fixed $x$.
Let $\Phi_t(a) = \Phi(a, x_t)$ and $\cL_t(a) = \cL(a, x_t)$, which are both continuous since $[k]$ has the discrete topology.
Then let $\sigma_t : \PiD \to \Sigma$ and $a_t : \PiD \to [k]$ be defined inductively by
\begin{align*}
a_t(\pi) = \pi(\sigma_1(\pi),\ldots,\sigma_{t-1}(\pi)) \quad \text{and} \quad
\sigma_t(\pi) = \Phi_t(a_t(\pi))\,.
\end{align*}
Writing the definition of the regret,
\begin{align*}
\Reg_n(\pi, x) = \sum_{t=1}^n \cL_t(a_t(\pi)) - \min_{a \in [k]} \sum_{t=1}^n \cL_t(a)\,.
\end{align*}
The second term is constant and, as we mentioned already, $a \mapsto \cL_t(a)$ is continuous. So it remains to check that $a_t$ is continuous for each $t$.
This follows by induction. The definition of the product topology means that for any fixed $\sigma_1,\ldots,\sigma_{t-1}$ and $b \in [k]$, the set 
\begin{align*}
U_b(\sigma_1,\ldots,\sigma_{t-1}) = \{\pi : \pi(\sigma_1,\ldots,\sigma_{t-1}) = b\}
\end{align*}
is open in $\PiD$. Let $\varepsilon$ denote the empty tuple. Then $a_1^{-1}(b) = U_b(\varepsilon)$ is open in $\PiD$. We confirm that $a_2$ is continuous and leave the rest
to the reader. That $a_2$ is continuous follows by writing 
\begin{align*}
a_2^{-1}(c) = \bigcup_{b=1}^k \left(U_b(\varepsilon) \cap U_c(\Phi_1(b))\right)\,.
\end{align*}
Hence $\pi \mapsto \Reg_n(\pi, x)$ is continuous and also measurable with respect to the Borel $\sigma$-algebra on $\PiD$.
Then let $f(\mu, \nu)$ be given by 
\begin{align*}
f(\mu, \nu) 
&= \int_{\PiD} \int_{\cX^n} \Reg_n(\pi, x) d\nu(x) d\mu(\pi) 
= \int_{\cX^n} \int_{\PiD} \Reg_n(\pi, x) d\mu(\pi) d\nu(x)\,, 
\end{align*}
where the exchange of integrals is justified by Fubini's theorem, which is applicable because the regret is bounded in $[-n,n]$.
Clearly $f$ is linear in both arguments.
We now claim that  both $\mu \mapsto f(\mu, \nu)$ and $\nu \mapsto f(\mu, \nu)$ are continuous. 
To see that $\nu \mapsto f(\mu,\nu)$ is continuous, note that $x \mapsto \int_{\PiD} d\mu(\pi) \Reg_n(\pi,x)$ is a $\cX^n \to [-n,n]$ continuous map owning to the choice of the discrete topology on $\cX^n$. Since $\sQ\subset \sP_r(\cX^n)$ is equipped with the weak*-topology, this implies the continuity of  $\nu \mapsto f(\mu,\nu)$.
The argument for the continuity of $\mu \mapsto f(\mu, \nu)$ is similar:
In particular, first note that $\pi \mapsto \int_{\cX^n} d\nu(x) \Reg_n(\pi,x)$ is a $\PiD \to [-n,n]$ continuous map, since owning to the choice of $\sQ$, the integral with respect to $\nu$ is a finite sum, and we have already established that for $x\in \cX^n$ fixed, $\pi \mapsto \Reg_n(\pi,x)$ is a $\PiD \to [-n,n]$ continuous map. Again, the choice of the weak*-topology on $\sP_r(\PiD)$ implies the desired continuity.

The final step is to note that for each policy $\mu \in \sP_r(\PiD)$ there exists a policy $\pi \in \Pi$ such that for all $x \in \cX^n$,
\begin{align*}
\Reg_n(\pi, x) = \int_{\PiD} \Reg_n(\pi_d, x) d\mu(\pi_d)\,.
\end{align*}
In particular, it is not hard to show that $\pi$ can be defined through $\pi(a_1,\phi(a_1,x),\dots,a_t,\phi(a_t,x))_a = \bbP_{\mu,x}(A_{t+1}=a|A_1=a_1,\dots,A_t=a_t)$, where $\bbP_{\mu,x}$ is the distribution resulting from using $\mu$ on the environment $x$. 
Here, the right-hand side is well defined (as a completely regular measure) 
because of the choice of $\cA$. 
That $\pi$ is well defined and is suitable follows from the definitions.
Putting things together,
\begin{align*}
\Reg^*_n
&= \inf_{\pi \in \Pi} \sup_{x \in \cX^n} \Reg_n(\pi, x) 
\leq \min_{\mu \in \sP_r(\PiD)} \sup_{x \in \cX^n} \int_{\PiD} \Reg_n(\pi, x) d\mu(\pi) 
\stackrel{(a)}\le \min_{\mu \in \sP_r(\PiD)} \sup_{\nu \in \sQ} f(\mu, \nu) \\ 
&\stackrel{(b)}= \sup_{\nu \in \sQ} \min_{\mu \in \sP_r(\PiD)} f(\mu, \nu) 
\stackrel{(c)}= \sup_{\nu \in \sQ} \min_{\pi \in \PiD} \int_{\cX^n} \Reg_n(\pi, x) d\nu(x) \\ 
&\stackrel{(d)}= \sup_{\nu \in \sQ} \min_{\pi \in \Pi} \int_{\cX^n} \Reg_n(\pi, x) d\nu(x)
\stackrel{(e)}= \BReg^*_n(\sQ)\,, \numberthis \label{eq:minmax1}
\end{align*}
where in 
(a) we used the fact that the Dirac measures are in $\sQ$,
(b) follows from Sion's theorem.
In  (c) we used the fact that the Dirac measures in $\sP_r(\PiD)$ are minimisers of $f(\cdot, \nu)$ for any $\nu$, 
in (d) we used that $\PiD\subset \Pi$ and, via a dynamic programming argument, that the deterministic policies from $\PiD$ minimise the Bayesian regret.
For (e), let $\bbP_{\pi\nu}$ be the joint induced by $\pi$ and $\nu$ over $[k]^n \times \cX^n$, $\E_{\pi\nu}$ the corresponding expectation
and define $\reg_n(a,x) = \sum_{t=1}^n \cL(a_t,x_t) - \min_{b\in [k]} \sum_{t=1}^n \cL(b,x_t)$.
Then, note that $\bbP_{\pi\nu}$ almost surely, $\E_{\pi\nu}[ \reg_n(A,X)|X] = \Reg_n(\pi,X)$, and thus, by the tower rule and because $\bbP_{\pi\nu,X} = \nu$ by assumption, 
$ \int_{\cX^n} \Reg_n(\pi, x) d\nu(x) = \BReg_n(\pi,\nu)$.
That $\BReg^*_n(\sQ) \leq \Reg^*_n$ follows from 
\begin{align*}
\Reg^*_n &= \inf_{\pi\in \Pi} \sup_{x \in \cX^n} \Reg_n(\pi, x) 
 = \inf_{\pi\in \Pi} \sup_{\nu\in \sQ}\int_{\cX^n} \Reg_n(\pi, x) d\nu(x) 
 \geq  \sup_{\nu\in \sQ}\inf_{\pi\in \Pi} \int_{\cX^n} \Reg_n(\pi, x) d\nu(x)\\
& =\BReg^*_n(\sQ)\,, \numberthis \label{eq:minmax2}
\end{align*}
where 
the second equality used that  for any fixed $\pi \in \Pi$,
$\nu\mapsto \int_{\cX^n} \Reg_n(\pi, x) d\nu(x)$ is a linear functional on $\sP_r(\cX^n)$,
which is thus maximised in the extreme points of $\sP_r(\cX^n)$, which are all the Dirac measures over $\cX^n$.
Combining \cref{eq:minmax1,eq:minmax2} gives the desired result.

\section{Proof of \cref{lem:finite-lem}}

Using the fact that the total variation distance is upper bounded 
by the Hellinger distance \citep[Lemma 2.3]{Tsy08} and the first inequality in \cref{eq:pinsker},
\begin{align}
\E_t[\Delta_t]
&= \sum_{a : P_{ta}^* > 0} P_{ta}^* \left(\E_t[X_{ta}] - \E_t[X_{ta} \mid A^* = a]\right) \nonumber \\
&\leq \sum_{a : P_{ta}^* > 0} P_{ta}^* \sqrt{\int_{[0,1]} \left(1 - \sqrt{\frac{d\bbP_{t,X_{ta}|A^* = a}}{d\bbP_{t,X_{ta}}}}\right)^2 d\bbP_{t,X_{ta}}} \label{eq:hel1} \\
&\leq \sqrt{k^{1/2} \sum_{a : P_{ta}^* > 0} (P_{ta}^*)^{3/2} \int_{[0,1]} \left(1 - \sqrt{\frac{d\bbP_{t,X_{ta}|A^* = a}}{d\bbP_{t,X_{ta}}}}\right)^2 d\bbP_{t,X_{ta}}}\,. \label{eq:hel2} 
\end{align}
\cref{eq:hel1} is true because the total variation distance is upper bounded by the Hellinger distance.
\cref{eq:hel2} uses Cauchy-Schwarz and the fact that $\sum_{a=1}^k (P_{ta}^*)^{1/2} \le k^{1/2}$, which also follows from Cauchy-Schwarz.
The next step is to apply Bayes law to the square root term. There are no measurability problems because both $X_{ta}$ and $A^*$ live in Polish spaces \citep{GV17}.
\begin{align*}
  &\int_{[0,1]} \left(1 - \sqrt{\frac{d\bbP_{t,X_{ta} | A^* = a}}{d\bbP_{t, X_{ta}}}}\right)^2 d\bbP_{t,X_{ta}} 
  = \int_{[0,1]} \left(1 - \sqrt{\frac{\bbP_t(A^* = a \mid X_{ta})(x)}{\bbP_t(A^* = a)}}\right)^2 d\bbP_{t,X_{ta}}(x) \\
  &\qquad\qquad= \E_t\left[\left(1 - \sqrt{\frac{\bbP_t(A^* = a \mid X_{ta})}{\bbP_t(A^* = a)}}\right)^2\right] \\
  &\qquad\qquad= \frac{1}{\sqrt{\bbP_t(A^* = a)}} \E_t\left[\frac{\left(\sqrt{\bbP_t(A^* = a)} - \sqrt{\bbP_t(A^* = a \mid X_{ta})}\right)^2}{\sqrt{\bbP_t(A^* = a)}}\right]\,.
\end{align*}
Substituting the above into \cref{eq:hel2} and using the fact that $P_{ta} = P_{ta}^* = \bbP_t(A^* = a)$ yields 
\begin{align*}
\E_t[\Delta_t] 
&\leq \sqrt{k^{1/2} \sum_{a : P_{ta} > 0} P_{ta} \E_t\left[\frac{\left(\sqrt{\bbP_t(A^* = a)} - \sqrt{\bbP_t(A^* = a \mid X_{ta})}\right)^2}{\sqrt{\bbP_t(A^* = a)}}\right]}  \\
&\leq \sqrt{k^{1/2} \sum_{a : P_{ta} > 0} P_{ta} \E_t\left[\sum_{c : P^*_{tc} > 0} \frac{\left(\sqrt{\bbP_t(A^* = c)} - \sqrt{\bbP_t(A^* = c \mid X_{ta})}\right)^2}{\sqrt{\bbP_t(A^* = c)}}\right]}\,, 
\end{align*}
where the second inequality follows by introducing the sum over $c$. Finally, note that
\begin{align*}
D_F(p, q) = \sum_{c : p_c \neq q_c } \frac{\left(\sqrt{p_c} - \sqrt{q_c}\right)^2}{\sqrt{q_c}} \,.
\end{align*}
The result follows from a direct computation using the independence of $A_t$ and $X_t$ under $\bbP_t$ (Lemma~\ref{lem:fundamental}).

\section{Cops and robbers}

To further demonstrate the flexibility of the approach we consider this special case of bandits with graph feedback.
In cops and robbers the learner observes the losses associated with all actions except the played action.
Except for constant factors, this problem is no harder than the full information setting where all losses are observed.
Cops and robbers is formalised in the partial monitoring framework by choosing $\Sigma = [0,1]^{k-1}$, $\cX = [0,1]^k$, $\cL(a, x) = x_a$ and 
\begin{align*}
\Phi(a, x) = (x_1,\ldots,x_{a-1},x_{a+1},\ldots,x_k)\,.
\end{align*}

\begin{theorem}\label{thm:cops-and-robbers}
The minimax regret of cops and robbers satisfies $\Reg^*_n \leq \sqrt{2n \log(k)}$.
\end{theorem}

This improves on the result by \cite{ACDK15} that $\Reg^*_n \leq 5 \sqrt{n \log(k)}$. 
We leave for the future the interesting question of whether or not this method recovers other known results for bandits with graph feedback.
\cref{thm:cops-and-robbers} follows immediately from Theorems~\ref{thm:hammer} and \ref{thm:swap}, and the following lemma.

\begin{lemma}\label{lem:cops-and-robbers}
Thompson sampling for cops and robbers satisfies $\E_t[\Delta_t] \leq \sqrt{2 I_t(A^* ; \Phi_t(A_t), A_t)}$ almost surely for all $t$.
\end{lemma}

\begin{proof}
Fix $t\in [n]$ and let $G_t = \argmax_a P_{ta}$.
Here, we assume that we have already discarded a suitable set of measure zero, so that we do not need to keep repeating the qualification `almost surely'.
Then, subtracting and adding $\cL_t(G_t)$, expanding the definitions and using that $P_t^* = P_t$,
\begin{align*}
&\E_t[\Delta_t]
= \sum_{a \neq G_t} P_{ta} \E_t[\cL_t(a) - \cL_t(G_t)] + \sum_{a \neq G_t} P_{ta} \E_t[\cL_t(G_t) - \cL_t(a) \mid A^* = a] \\
&\leq \sum_{a \neq G_t} P_{ta} \left(\sqrt{\frac{1}{2} \bKL{\bbP_{t,\cL_t(G_t) | A^* = a}}{\bbP_{t,\cL_t(G_t)})}} + \sqrt{\frac{1}{2} \bKL{\bbP_{t,\cL_t(a) | A^* = a}}{\bbP_{t,\cL_t(a)}}}\right) \\
&\leq \sqrt{\textrm{(A)}} + \sqrt{\textrm{(B)}}\,,
\end{align*}
where the first inequality follows from grouping the terms that involve $\cL_t(G_t)$ and those that involve $\cL_t(a)$ and then using Pinsker's inequality (\ref{eq:pinsker}),
while the second follows from Cauchy-Schwarz and the definitions,
\begin{align*}
\textrm{(A)} 
&= \frac{1 - P_{tG_t}}{2} \sum_{a \neq G_t} P_{ta} \bKL{\bbP_{t,\cL_t(G_t) | A^* = a}}{\bbP_{t,\cL_t(G_t)}}\,, \\ 
\textrm{(B)}
&= \frac{1 - P_{tG_t}}{2} \sum_{a \neq G_t} P_{ta} \bKL{\bbP_{t,\cL_t(a) | A^* = a}}{\bbP_{t,\cL_t(a)}}\,. 
\end{align*}
The result is completed by bounding each term separately. Using that $1-P_{tG_t} = \sum_{b\ne G_t} P_{tb}$,
\begin{align*}
\textrm{(A)}
&=\frac{1 - P_{tG_t}}{2} \sum_{a \neq G_t} P_{ta} \bKL{\bbP_{t,\cL_t(G_t) | A^* = a}}{\bbP_{t,\cL_t(G_t)}} \\
&= \frac{1}{2} \sum_{a \neq G_t} P_{ta} \sum_{b \neq G_t} P_{tb} \bKL{\bbP_{t,\cL_t(G_t) | A^* = a}}{\bbP_{t,\cL_t(G_t)}} \\
&\leq \frac{1}{2} \sum_{a \neq G_t} P_{ta} \sum_{b \neq G_t} P_{tb} \bKL{\bbP_{t,\Phi_t(b) | A^* = a}}{\bbP_{t,\Phi_t(b)}} \\
&\leq \frac{1}{2} I_t(A^* ; \Phi_t(A_t), A_t)\,,
\end{align*}
where the first inequality follows from the data processing inequality (for $b\ne G_t$, $\cL_t(G_t)$ is a deterministic function of $\Phi_t(b)$)
and the last from Lemma~\ref{lem:fundamental}.
The second term is bounded in almost the same way. Here we use the fact that $1 - P_{tG_t} \leq 1 - P_{ta}$ for all $a \in [k]$:
\begin{align*}
\textrm{(B)} 
&=\frac{1 - P_{tG_t}}{2} \sum_{a \neq G_t} P_{ta} \bKL{\bbP_{t,\cL_t(a) | A^* = a}}{\bbP_{t,\cL_t(a)}} \\
&\leq \frac{1}{2} \sum_{a \neq G_t} (1 - P_{ta}) P_{ta} \bKL{\bbP_{t,\cL_t(a) | A^* = a}}{\bbP_{t,\cL_t(a)}} \\
&= \frac{1}{2} \sum_{a \neq G_t} P_{ta} \sum_{b \neq a} P_{tb} \bKL{\bbP_{t,\cL_t(a) | A^* = a}}{ \bbP_{t,\cL_t(a)}} \\
&\leq \frac{1}{2} \sum_{a \neq G_t} P_{ta} \sum_{b \neq a} P_{tb} \bKL{\bbP_{t,\Phi_t(b) | A^* = a}}{ \bbP_{t,\Phi_t(b)}} \\
&\leq \frac{1}{2} I_t(A^* ; \Phi_t(A_t), A_t)\,.
\end{align*}
Combining the previous displays and rearranging completes the proof.
\end{proof}

\section{Proof of \cref{thm:local}}\label{app:degenerate}

We need the following lemma, which characterises actions $c \in \cN_{ab}$ as having loss vectors $\ell_c$ that are convex combinations of $\ell_a$ and $\ell_b$.

\begin{lemma}[\citealt{BFPRS14}]
For all actions $c \in \cN_{ab}$ there exists an $\alpha \in [0,1]$ such that $\ell_c = \alpha \ell_a + (1 - \alpha) \ell_b$.
\end{lemma}

\begin{proof}[\cref{thm:local}]
In order to define the algorithm we first choose a subset $\cC \subseteq [k]$ such that $\cC$ contains no duplicate or degenerate actions
and $\cup_{c \in \cC} C_c = \Delta^{k-1}$.
We assume additionally that $P_t^*$ is constant on duplicate actions.
Construct the parent function $\cP_t$ on actions in $\cC$ in the same way as Mario sampling.
For $a \neq b$ let $\cT_{ab} = (c_1,\ldots,c_m)$ be an ordering 
of 
\begin{align*}
\{c \in ([k] \setminus \cC) \cup\{b\} : c = b \text{ or exists } \alpha \in (0,1] \text{ with } \ell_c = \alpha \ell_a +  (1 - \alpha) \ell_b\}
\end{align*}
ordered by decreasing $\alpha$ values and with $c_m = b$. In other words $\cT_{ab}$ is a sequence of actions starting with duplicates of $a$, then actions $c$ for which $\ell_c$ is a
strict convex combination of $\ell_a$ and $\ell_b$, with actions that are `closer' to $a$ sorted first. The last element of $\cT_{ab}$ is $b$ itself. Duplicates of $b$ are not included in $\cT_{ab}$.
Let $\cT_{aa}$ be the duplicates of $a$, excluding $a$, in an arbitrary order. 
Then define
\begin{align*}
\cP'_t(c) &= \begin{cases}
\cT_{c\cP(c)}[1] & \text{if } c \in \cC \setminus \{G_t\} \\
\cT_{cc}[1] & \text{if } c = G_t \text{ and } \cT_{cc} \neq \emptyset \\
\cT_{ab}[i+1] & \text{if } c = \cT_{ab}[i]\,.
\end{cases}
\end{align*}
Let $W_t'$ be the water transfer operator using the tree generated by $\cP'_t$ instead of $\cP_t$ and $P_t = (W_t')^k P_t^*$.
Now we follow the proof of \cref{thm:pm}. Let $t \in [n]$ be fixed. We start by bounding $\E_t[\Delta_t]$ in terms of the expected information gain.
Given $b \in \cC \setminus \{G_t\}$ let $f_b : \cN_{a\cP_t(a)} \to \R$ be a function with
\begin{align*}
\sum_{c \in \cN_{b\cP_t(b)}} f_b(c, \Phi_t(c)) = \cL_t(b) - \cL_t(\cP_t(b))\,,
\end{align*}
which exists by the definition of local observability. By definition we may assume that $\norm{f_b}_\infty \leq v$. Then
\begin{align*}
\E_t[\cL_t(a)] = \E_t\left[\cL_t(G_t) + \sum_{b \in \cA_t(a) \setminus \{G_t\}} \sum_{c \in \cN_{b\cP_t(b)}} f_b(c, \Phi_t(c))\right]\,.
\end{align*}
Therefore
\begin{align*}
\E_t[\Delta_t] 
&\leq v\sum_{a=1}^k P_{ta}^* \sum_{b \in \cA_t(a) \setminus \{G_t\}} \sum_{c \in \cN_{b \cP_t(b)}} \sqrt{2 \bKL{\bbP_{t,\Phi_t(c)|A^* = a}}{\bbP_{t,\Phi_t(c)}}} \\
&\leq vk \sqrt{4\sum_{a=1}^k P_{ta}^* \sum_{b \in \cA_t(a) \setminus \{G_t\}} \sum_{c \in \cN_{b \cP_t(b)}} P_{ta}^* \bKL{\bbP_{t,\Phi_t(c)|A^* = a}}{\bbP_{t,\Phi_t(c)}}} \\
&\leq vk^{3/2} \sqrt{4\sum_{a=1}^k P_{ta}^* \sum_{b \in \cA_t(a) \setminus \{G_t\}} \sum_{c \in \cN_{b \cP_t(b)}} P_{tc} \bKL{\bbP_{t,\Phi_t(c)|A^* = a}}{\bbP_{t,\Phi_t(c)}}} \\
&\leq vk^{3/2} \sqrt{8 \sum_{a=1}^k P_{ta}^* \sum_{c=1}^k P_{tc} \bKL{\bbP_{t,\Phi_t(c)|A^* = a}}{\bbP_{t,\Phi_t(c)}}} \\
&= vk^{3/2} \sqrt{8 I_t(A^* ; \Phi_t(A_t), A_t)}\,.
\end{align*}
And the result follows from \cref{thm:hammer} and \cref{thm:swap}.
\end{proof}

\section{Proof of \cref{thm:pm:hard}}\label{app:global}

Again Thompson sampling does not explore sufficiently often. 
The most straightforward correction is to simply add a small amount of forced exploration, which
was also used in combination with Exp3 in prior analysis of these games \citep{CBLuSt06}.
We let
\begin{align}
P_t = (1 - \gamma) P_t^* + \gamma \ones / k\,,
\label{eq:pm:global-policy}
\end{align}
where ties in the $\argmax$ that defines $P_t^*$ are broken by prioritising Pareto optimal actions, which means that $P_{ta}^* = 0$ for all degenerate actions.
As usual, the crucial step is to bound the expected $1$-step regret in terms of the information gain.

\begin{lemma}
For the policy playing according to \cref{eq:pm:global-policy} it holds almost surely that
\begin{align*}
\E_t[\Delta_t] \leq \gamma + kv \sqrt{\frac{2 I_t(A^* ; \Phi_t(A_t), A_t)}{\gamma}}\,.
\end{align*}
\end{lemma}

\begin{proof}
Let $a_{\circ}$ be an arbitrary fixed Pareto optimal action and for 
each Pareto optimal action $a$ let $f_a : [k] \times \Sigma \to \R$ be a function with $\norm{f_a}_\infty \leq v$ such that
\begin{align*}
\sum_{c=1}^k f_a(c, \Phi(c, x)) = \cL(a, x) - \cL(a_{\circ}, x) \quad \text{for all } x \in [d]\,.
\end{align*}
The next step is to decompose the expected loss in terms of $f$:
\begin{align*}
\E_t[\Delta_t]
&= \sum_{a=1}^k P_{ta} \E_t[\cL_t(a)] - \sum_{a=1}^k P_{ta}^* \E_t[\cL_t(a) \mid A^* = a] \\
&\leq \gamma + \sum_{a=1}^k P_{ta}^*\left( \E_t[\cL_t(a)] - \E_t[\cL_t(a) \mid A^* = a]\right) 
\\
&= \gamma + \sum_{a=1}^k P_{ta}^*\left( \E_t[\cL_t(a) - \cL_t(a_{\circ})] - \E_t[\cL_t(a) - \cL_t(a_{\circ}) \mid A^* = a]\right) \\
&= \gamma + \sum_{a=1}^k P_{ta}^*\left( \E_t\left[\sum_{c=1}^k f_a(c, \Phi_t(c)) \right] - \E_t\left[\sum_{c=1}^k f_a(c, \Phi_t(c)) \,\middle|\, A^* = a\right]\right)\,,
\end{align*}
where the inequality follows from the definition of $P_t$ and the fact that losses are bounded in $[0,1]$.
Then, by Pinsker's inequality (\ref{eq:pinsker}),
\begin{align*}
\E_t[\Delta_t]
&\leq \gamma + v \sum_{c=1}^k \sum_{a=1}^k P_{ta}^* \sqrt{2 \bKL{\bbP_{t,\Phi_t(c)|A^* = a}}{\bbP_{t,\Phi_t(c)}}} \\
&\leq \gamma + v \sqrt{2k \sum_{a=1}^k P_{ta}^* \sum_{c=1}^k \bKL{\bbP_{t,\Phi_t(c)|A^* = a}}{\bbP_{t,\Phi_t(c)}}} \\
&\leq \gamma + k v \sqrt{\frac{2}{\gamma} \sum_{a=1}^k P_{ta}^* \sum_{c=1}^k P_{tc} \bKL{\bbP_{t,\Phi_t(c)|A^* = a}}{\bbP_{t,\Phi_t(c)}}} \\
&= \gamma + kv \sqrt{\frac{2 I_t(A^* ; \Phi_t(A_t), A_t)}{\gamma}}\,,
\end{align*}
where the first inequality follows from Pinsker's inequality (\ref{eq:pinsker}), the second from Cauchy-Schwarz, the third because $1 \leq k P_{tc}/\gamma$ for all $c$.
The last term follows from Lemma~\ref{lem:fundamental}.
\end{proof}

\begin{proof}[Theorem~\ref{thm:pm:hard}]
By the previous lemma and Corollary~\ref{cor:general},
\begin{align*}
\BReg_n \leq n \gamma + kv \sqrt{\frac{2n \log(k)}{\gamma}} \leq 3(nkv)^{2/3} (\log(k)/2)^{1/3}\,,
\end{align*}
where we choose $\gamma = n^{-1/3} (kv)^{2/3} (\log(k)/2)^{1/3}$ and note that when $\gamma > 1$ the claim in the theorem is immediate.
\end{proof}

\section{The water transfer operator}\label{app:lem:pm:P}

Here we explain in more detail the water transfer operator defined by \cref{fig:flow} and provide the proof of Lemma~\ref{lem:pm:P}.
An example with $k = 6$ is illustrated below.

\begin{figure}[H]
\begin{center}
\definecolor{ocean}{HTML}{316ac6}
\definecolor{flow}{HTML}{d14108}
\newcommand{\drawconnector}{
\path[fill=ocean!50!white] (-0.03,0.1) rectangle (1.03,0.2);
\draw (0,0.2) -- (1,0.2);
\draw (0,0.1) -- (1,0.1);
}

\newcommand{\drawrconnector}[2]{
\begin{scope}[shift={(#1,0)}]
\draw[#2,-{Latex[scale=1.2]}] (0.2,-0.1) -- (0.8,-0.1);
\drawconnector{}
\end{scope}
}

\newcommand{\drawlconnector}[2]{
\begin{scope}[shift={(#1,0)}]
\draw[#2,-{Latex[scale=1.2]}] (0.8,-0.1) -- (0.2,-0.1);
\drawconnector{}
\end{scope}
}

\newcommand{\drawbucket}[2]{
\begin{scope}[shift={(#1,0)}]
\path[fill=ocean!50!white] (0,0) rectangle (1,#2);
\begin{scope}[shift={(0,#2)}];
\draw[white] (0,0) sin (0.125,-0.05) cos (0.25,0) sin (0.375,-0.05) cos (0.5,0) sin (0.625,-0.05) cos (0.75,0) sin (0.875,-0.05) cos (1,0);
\begin{scope}[shift={(0,0.035)}]
\draw[white] (0,0) sin (0.125,-0.05) cos (0.25,0) sin (0.375,-0.05) cos (0.5,0) sin (0.625,-0.05) cos (0.75,0) sin (0.875,-0.05) cos (1,0);
\end{scope}
\end{scope}

\draw (0,0) rectangle (1,2);
\draw[white,line width=3pt] (0.02,2) -- (0.98,2);
\draw (1,1.5) to [out=0,in=90] (1.5,1.125) to [out=-90, in=0 ] (1,0.75);
\draw (1,1.58) to [out=0,in=90] (1.58,1.125) to [out=-90, in=0 ] (1,0.67);
\end{scope}
}

\begin{tikzpicture}[thick,scale=0.8]
\tikzstyle{water} = [draw=none,fill=ocean]
\node[anchor=west] at (-3.5,1) {$P$};
\drawbucket{-2}{0.6}
\drawbucket{0}{1.7}
\drawbucket{2}{1.2}
\drawbucket{4}{1.8}
\drawbucket{6}{0.4}
\drawbucket{8}{0.8}
\drawrconnector{-1}{lightgray}
\drawrconnector{1}{black}
\drawrconnector{3}{lightgray}
\drawrconnector{5}{lightgray}
\drawlconnector{7}{lightgray}
\end{tikzpicture}

\vspace{0.2cm}
\begin{tikzpicture}[thick,scale=0.8]
\tikzstyle{water} = [draw=none,fill=ocean]
\node[anchor=west] at (-3.5,1) {$W_tP$};
\drawbucket{-2}{0.6}
\drawbucket{0}{1.45}
\drawbucket{2}{1.45}
\drawbucket{4}{1.8}
\drawbucket{6}{0.4}
\drawbucket{8}{0.8}
\drawrconnector{-1}{lightgray}
\drawrconnector{1}{black}
\drawrconnector{3}{black}
\drawrconnector{5}{black}
\drawlconnector{7}{lightgray}
\end{tikzpicture}

\vspace{0.2cm}
\begin{tikzpicture}[thick,scale=0.8]
\tikzstyle{water} = [draw=none,fill=ocean]
\node[anchor=west] at (-3.5,1) {$W_t^2P$};
\drawbucket{-2}{0.6}
\drawbucket{0}{1.275}
\drawbucket{2}{1.275}
\drawbucket{4}{1.275}
\drawbucket{6}{1.275}
\drawbucket{8}{0.8}
\drawrconnector{-1}{lightgray}
\drawrconnector{1}{lightgray}
\drawrconnector{3}{lightgray}
\drawrconnector{5}{lightgray}
\drawlconnector{7}{lightgray}

\node at (-1.5,-0.5) {$a$};
\node at (0.5,-0.5) {$b$};
\node at (2.5,-0.5) {$c$};
\node at (4.5,-0.5) {$d$};
\node at (6.5,-0.5) {$e = G_t$};
\node at (8.5,-0.5) {$f$};
\end{tikzpicture}
\end{center}
\caption{Water transfer process}\label{fig:water}
\end{figure}
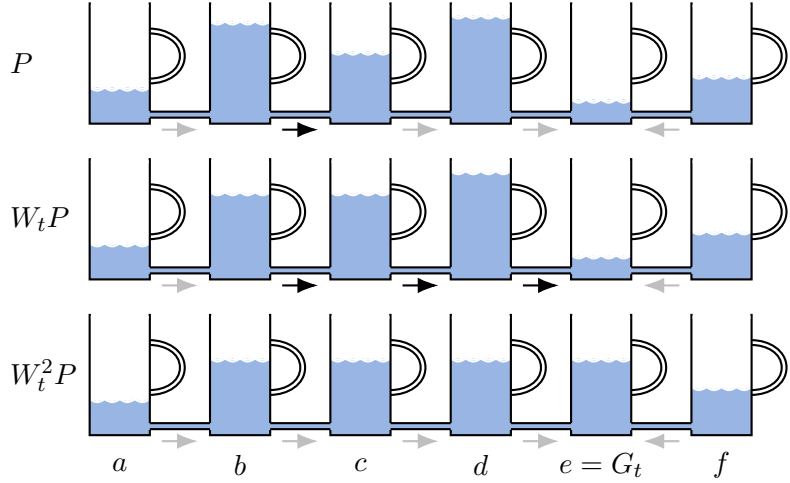

The mugs correspond to actions and are connected at the bottom with valves that default to being closed. The total volume of water sums to $1$.
Arrows correspond to edges in the tree. The dark arrows indicate which valves are open in each iteration and show the direction of flow.
In the first application of $W_t$, mug $c$ is anomalous and the water in mugs $b$ and $c$ is averaged. Imagine opening the valve connecting $b$ and $c$. The water in $a$ is too low to be included
in the average. In the second application, the water in mugs $b$, $c$, $d$ and $e$ is averaged.
Further applications of $W_t$ have no effect because there are no anomalous actions.

\begin{remark}
Another way to think about the application of $W_t$ to $P$ is as follows. First the anomalous action $a$ is identified, if it exists.
Then water flows continuously into $a$ from the set of descendants of $a$ that contain more water than $a$ until $a$ is no longer anomalous. 
\end{remark}

\begin{proof}[Lemma~\ref{lem:pm:P}]
To begin, notice that every application of the water transfer operator reduces the number of anomalous actions by at least one because:
(1) If $a$ is selected by \cref{fig:flow} then $a$ is not anomalous in $W_t P$ and (2) only actions that were anomalous in $P$ can be anomalous in $W_t P$. 
Since there are at most $k$ anomalous actions in any $P$, the water transfer operator
ceases to have any affect after more than $k$ operations. Hence $Q_a \geq \max_{b \in \cD_t(a)} Q_b$ for all $a$ and the second part follows.
For the first part we show that the loss of $W_t P$ is always smaller than $P$.
Let $\bar L(b)=\E_t[\cL_t(b)]$ for $b\in [k]$ and $a\in [k]$ be the anomalous action in $P$ selected by the algorithm. 
Then let $\cC = \{ b\in [k]\,: (W_tP)_b \ne P_b \}$ be the set of actions for which the distribution is changed.
By the definition of the tree, $\bar L(b) \ge \bar L(a)$ for any $b\in \cC$,
\begin{align*}
\sum_{b=1}^k (P_b - (W_tP)_b )\bar L(b)
&= (P_a - (W_t P)_a) \bar L(a) + \sum_{b \in \cC,b\ne a} (P_b-(W_tP)_b) \bar L(b)  \\
&= \sum_{b \in \cC,b\ne a} (P_b-(W_t P)_b) (\bar L(b) - \bar L(a)) \ge 0\,,
\end{align*}
which shows that $W_t$ decreases the expected loss.
For the last part, notice that during each iteration of the water transfer operator the update occurs by averaging the contents of a number of mugs
so that all have the same level (\cref{fig:water}). Once a group of mugs have been averaged together, subsequently they are always averaged together. 
It follows that after every iteration the actions $[k]$ can be partitioned so that the level in each partition is the average of $P_a$. 
Suppose that $a$ is in partition $S \subseteq [k]$. Then $Q_a = \frac{1}{|S|} \sum_{b \in S} P_b \geq P_a / k$.
\end{proof}

\section{Failure of Thompson sampling for partial monitoring}\label{app:ts-fail}

The following example with $k = 3$ and $d = 2$ illustrates the failure of Thompson sampling for locally observable non-degenerate partial monitoring games.
The game is a toy `spam filtering' problem where the learner can either classify an email as spam/not spam or pay a small cost for the true label.
The functions $\Phi$ and $\cL$ are represented by the tables below, with the learner choosing the rows and adversary the columns.
\begin{figure}[H]
\centering
\renewcommand{\arraystretch}{1.5}
\small
\begin{tabular}{|l|ll|}
\hline
\textbf{Losses} $\bm \cL$      & \cellcolor{lightgray}\textsc{not spam} & \cellcolor{lightgray}\textsc{spam}  \\ \hline

\cellcolor{lightgray}\textsc{spam}   &  1  & 0  \\
\cellcolor{lightgray}\textsc{not spam} & 0   & 1  \\
\cellcolor{lightgray}\textsc{unknown}  & c & c \\
\hline
\end{tabular}
\hspace{0.5cm}
\renewcommand{\arraystretch}{1.5}
\begin{tabular}{|l|cc|}
\hline
\textbf{Signals} $\bm\Phi$      & \cellcolor{lightgray}\textsc{not spam} & \cellcolor{lightgray}\textsc{spam} \\ \hline
\cellcolor{lightgray}\textsc{spam}   & $\bot$   & $\bot$  \\
\cellcolor{lightgray}\textsc{not spam} & $\bot$   & $\bot$  \\
\cellcolor{lightgray}\textsc{unknown}  & \textsc{not spam} & \textsc{spam} \\
\hline
\end{tabular}
\caption{The `spam' partial monitoring game. For $c < 1/2$ the game is locally observable and non-degenerate. For $c = 1/2$ the game is locally observable, but degenerate. For $c > 1/2$ the
game is not locally observable, but is globally observable. For $c = 0$ the game is trivial.}\label{fig:spam}
\end{figure}

The learner only elicits meaningful feedback in the spam game by paying a cost of $c$ to observe the true label.
For appropriately chosen $c$ and prior, we will see that Thompson sampling never chooses the revealing action, cannot learn, and hence suffers linear regret.
Let $c > 0$ and $\nu$ be the mixture of two Dirac's: $\nu = \frac{1}{2} \delta^n_{\textsc{spam}} + \frac{1}{2} \delta^n_{\textsc{not spam}}$, where $\delta^n_i$ is the Dirac measure on $(i,i,\ldots,i)$.
With these choices the optimal action is almost surely either \textsc{spam} or \textsc{not spam}. Since choosing these actions does not reveal any information, the posterior
is equal to the prior and Thompson sampling plays these two actions uniformly at random. Clearly this leads to linear regret relative to the optimal policy that plays the exploratory action once to identify
the adversary and plays optimally for the remainder.
Since this result holds for any strictly positive cost, it also shows that Thompson sampling does not work for globally observable games.

\section{Structural lemmas for partial monitoring}\label{app:structure}

\begin{lemma}\label{lem:structure}
Let $a,b \in [k]$ be distinct actions in a non-degenerate game and $u \in C_a$.
Then there exists an action $c \in \cN_b \setminus\{b\}$ such that $\ip{\ell_b - \ell_c, u} \geq 0$.
Furthermore, if $u \notin C_b$, then $\ip{\ell_b - \ell_c, u} > 0$.
\end{lemma}

\begin{proof}
Let $w$ be a point in the relative interior of $C_b$, which means that $\ip{\ell_b, w} < \min_{c \neq b} \ip{\ell_c, w}$. 
Now let $c \in \cN_b\setminus \{b\}$ be an action such that
$v = u + \alpha (w - u) \in C_b \cap C_c$ for some $\alpha \in [0,1)$, which exist because $C_b$ is closed convex set
and hence $\{u + \alpha(w-u) \,:\, \alpha \in \R \} \cap C_b$, which is nonempty, must be a closed segment.
Let $f(x) = \ip{\ell_b - \ell_c, u + x(w - u)}$.
By definition, $f(\alpha) = 0$ and $f(1) < 0$.
Since $f$ is linear it follows that $f(0) = \ip{\ell_b - \ell_c, u} \geq 0$.
The second part follows because if $u \notin C_b$, then $\alpha > 0$, which means that $f(0) > f(\alpha) = 0$.
\end{proof}

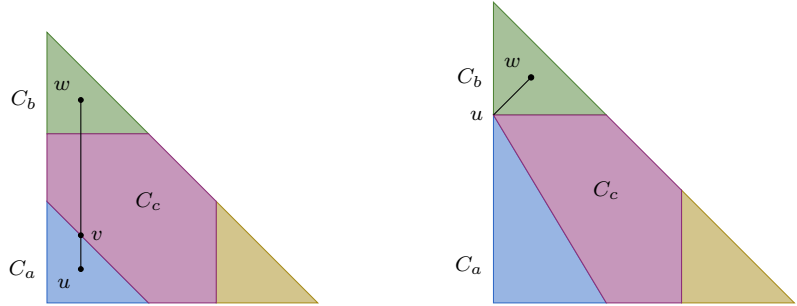
\begin{figure}[H]
\centering
\begin{tikzpicture}[font=\scriptsize,scale=0.9]
\coordinate (1) at (0,0);
\coordinate (2) at (4,0);
\coordinate (3) at (0,4);
\coordinate (4) at (1.5,0);
\coordinate (5) at (2.5,0);
\coordinate (6) at (2.5,1.5);
\coordinate (7) at (0,1.5);
\coordinate (8) at (0,2.5);
\coordinate (9) at (1.5,2.5);

\draw[fill=c1!50!white,draw=c1] (1) -- (4) -- (7) -- cycle;
\draw[fill=c2!50!white,draw=c2] (5) -- (6) -- (2) -- cycle;
\draw[fill=c3!50!white,draw=c3] (8) -- (3) -- (9) -- cycle;
\draw[fill=c4!50!white,draw=c4] (8) -- (9) -- (6) -- (5) -- (4) -- (7) -- cycle;

\draw[fill=black] (0.5,0.5) circle (1pt);
\draw[fill=black] (0.5,3) circle (1pt);
\draw[fill=black] (0.5,1) circle (1pt);
\draw (0.5,0.5) -- (0.5,3);
\node[anchor=south east] at (0.5,3) {$w$};
\node[anchor=north east] at (0.5,0.5) {$u$};
\node[anchor=west] at (0.5,1) {$v$};
\node[anchor=east] at (0,0.5) {$C_a$};
\node[anchor=east] at (0,3) {$C_b$};
\node at (1.5,1.5) {$C_c$};
\end{tikzpicture}
\hspace{1.5cm}
\begin{tikzpicture}[font=\scriptsize]
\coordinate (1) at (0,0);
\coordinate (2) at (4,0);
\coordinate (3) at (0,4);
\coordinate (4) at (1.5,0);
\coordinate (5) at (2.5,0);
\coordinate (6) at (2.5,1.5);
\coordinate (7) at (0,2.5);
\coordinate (8) at (0,2.5);
\coordinate (9) at (1.5,2.5);

\draw[fill=c1!50!white,draw=c1] (1) -- (4) -- (7) -- cycle;
\draw[fill=c2!50!white,draw=c2] (5) -- (6) -- (2) -- cycle;
\draw[fill=c3!50!white,draw=c3] (8) -- (3) -- (9) -- cycle;
\draw[fill=c4!50!white,draw=c4] (8) -- (9) -- (6) -- (5) -- (4) -- (7) -- cycle;

\draw[fill=black] (0.5,3) circle (1pt);
\draw (0,2.5) -- (0.5,3);
\node[anchor=south east] at (0.5,3) {$w$};
\node[anchor=east] at (0,2.5) {$u$};
\node[anchor=east] at (0,0.5) {$C_a$};
\node[anchor=east] at (0,3) {$C_b$};
\node at (1.5,1.5) {$C_c$};
\end{tikzpicture}

\caption{Illustration for the proof of Lemma~\ref{lem:structure}.
The bottom left region is $C_a$ and $u \in C_a$ so that $a$ minimises $\E_{x \sim u}[\cL(a, x)]$.
The lemma proves that for the situation in the left figure: $\E_{x \sim u}[\cL(c, x)] < \E_{x \sim u}[\cL(b, x)]$. The strict inequality is replaced by an equality if $u \in C_a \cap C_b$ as
in the right figure, when $u = v$.
}

\end{figure}

\newcommand{\ri}{\operatorname{ri}}
\newcommand{\sind}[1]{\mathds{1}_{#1}}

\begin{lemma}
\label{lem:structure2}
Consider a non-degenerate game and
let $u \in \Delta^{k-1}$ and $V = \{a : u \in C_a\}$ and $E = \{(a, b) \in V : a \text{ and } b \text{ are neighbours}\}$.
Then the graph $(V, E)$ is connected.
\end{lemma}

\begin{proof}
This must be a known result about the facet graph of convex polytopes.
We give a dimension argument. You may find \cref{fig:structure} useful.
Let $B_\epsilon(x) = \{y \in \Delta^{d-1} : \norm{y - x}_2 \leq \epsilon\}$,
$\cH_d$ be the $d$-dimensional Hausdorff measure and $\ri$ be the relative interior operator.
Since the cells are closed, there exists an $\epsilon > 0$ such that $B_\epsilon(u) \cap C_c = \emptyset$ for all $c \notin V$.
Then let $a^*(v) = \{a \in [k] : v \in C_a\}$ be the set of actions that are optimal at $v \in \Delta^{d-1}$.
It is easy to see that if $a^*(v) = \{a, b\}$ for some $v \in \Delta^{d-1}$, then $a$ and $b$ are neighbours.
Let $N = \{v \in \Delta^{d-1} : |a^*(v)| > 2\}$, which by the assumption that there are no duplicate/degenerate actions has dimension at most $d - 3$ and hence $\cH_{d-2}(N) = 0$.
Let $a,b \in V$ be distinct and $v,w \in B_\epsilon(u)$ be such that $B_\delta(v) \subset C_a$ and $B_\delta(w) \subset C_b$ for some $\delta > 0$, which by 
definition means that the interval $[v,w] \cap C_c = \emptyset$ for all $c \notin V$.
Let $A$ be the affine space containing $v$ with normal $v - w$ and 
$P = \{\argmin_{x \in A} \norm{x - y}_2 : y \in N\}$ be the projection of $N$ onto $A$. 
Since projection onto a plane cannot increase the Hausdorff measure, $\cH_{d-2}(P) = 0$. On the other hand, the fact that $A \cap B_\delta(v)$ has dimension $d-2$ means that $\cH_{d-2}(A \cap B_\delta(v)) > 0$.
Therefore
$\cH_{d-2}(B_\delta(v) \cap (A \setminus P)) > 0$ and hence there exists an $x \in B_{\delta}(v) \cap A$ and $y = x + w - v \in B_\delta(w)$ such that
$[x,y] \cap N = \emptyset$. Then the set $\cup_{z \in [x,y]} a^*(u)$ forms a connected path in $V$ between $a$ and $b$. 
\end{proof}

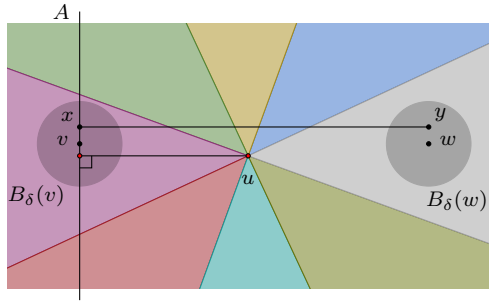
\begin{figure}[H]
\centering
\begin{tikzpicture}[font=\scriptsize,scale=1.6]
\begin{scope}
\clip (-2,-1.1) rectangle (2,1.1);
\begin{scope}[rotate=25]
\clip (0,0) circle (3cm);

\coordinate (0) at (0,0);
\coordinate (1) at (3,0);
\coordinate (2) at (3,3);
\coordinate (3) at (0,3);
\coordinate (4) at (-3,3);
\coordinate (5) at (-3,0);
\coordinate (6) at (-3,-3);
\coordinate (7) at (0,-3);
\coordinate (8) at (3,-3);
\coordinate (9) at (3,0);

\draw[fill=c1!50!white,draw=c1] (0) -- (1) -- (2) -- cycle;
\draw[fill=c2!50!white,draw=c2] (0) -- (2) -- (3) -- cycle;
\draw[fill=c3!50!white,draw=c3] (0) -- (3) -- (4) -- cycle;
\draw[fill=c4!50!white,draw=c4] (0) -- (4) -- (5) -- cycle;
\draw[fill=c5!50!white,draw=c5] (0) -- (5) -- (6) -- cycle;
\draw[fill=c6!50!white,draw=c6] (0) -- (6) -- (7) -- cycle;
\draw[fill=c7!50!white,draw=c7] (0) -- (7) -- (8) -- cycle;
\draw[fill=c8!50!white,draw=c8] (0) -- (8) -- (9) -- cycle;
\end{scope}
\end{scope}

\draw (-1.4,1.2) -- (-1.4,-1.2);
\node[anchor=east] at (-1.4,1.2) {$A$};
\draw (-1.4,0.24) -- (1.5,0.24);
\draw[] (0,0) -- (-1.4,0);
\draw[] (-1.3,0) -- (-1.3,-0.1) -- (-1.4,-0.1);

\draw[fill,opacity=0.2] (1.5,0.1) circle (10pt);
\draw[fill,opacity=0.2] (-1.4,0.1) circle (10pt);
\draw[fill] (-1.4,0.24) circle (0.5pt);
\draw[fill] (1.5,0.24) circle (0.5pt);
\draw[fill] (-1.4,0.1) circle (0.5pt);
\draw[fill] (1.5,0.1) circle (0.5pt);
\draw[fill=red] (-1.4,0) circle (0.5pt);
\draw[fill=red] (0,0) circle (0.5pt);

\node[anchor=south east,inner sep=2pt] at (-1.4,0.24) { $x$};
\node[anchor=south east,inner sep=4pt] at (-1.4,0) { $v$};
\node[anchor=south west,inner sep=2pt] at (1.5,0.24) { $y$};
\node[anchor=south west,inner sep=4pt] at (1.5,0) { $w$};
\node[anchor=south east,inner sep=4pt] at (-1.43,-0.51) { $B_\delta(v)$};
\node[anchor=south west,inner sep=2pt] at (1.43,-0.51) { $B_\delta(w)$};
\node[anchor=north,inner sep=7pt] at (0,0) { $u$};
\end{tikzpicture}
\caption{Illustration for the proof of Lemma~\ref{lem:structure2} when $d = 3$. The whole region shown is a subset of $B_\epsilon(u)$. The set $N$ in this case consists only of $u$, which has $1$-dimensional Hausdorff measure zero.
The cells crossed by the interval $[x,y]$ form the path between $a$ and $b$ in $V$. 
}
\label{fig:structure}
\end{figure}

\begin{proof}[Lemma~\ref{lem:tree}]
By definition there are no edges starting from $G_t$.
By Lemma~\ref{lem:structure}, for all $a \notin V_t$ there is a neighbour $b \in \cN_a$ with strictly smaller loss, $\E_t[\cL_t(b)] < \E_t[\cL_t(a)]$.
Hence the definition of $\cP_t(a)$ ensures there are no cycles and that every path starting from $a \notin V_t$ eventually leads to $V_t$.  
Then by Lemma~\ref{lem:structure} the graph $(V_t, E_t)$ is connected, which means that for $a \in V_t$ the parent $\cP_t(a)$ is a vertex $b \in V_t$ that is closest to $G_t$.
Hence all paths lead to $G_t$.
\end{proof}

The next two lemmas bound on the supremum norms of the estimation functions. The first is restricted to the non-degenerate case where the result was already known
and the second holds for all globally observable games. 

\begin{lemma}[\cite{LS18pm}, Lemma~9]\label{lem:sup-bound}
For locally observable non-degenerate games, 
the function $f$ in \cref{eq:est-def} can be chosen so that $\norm{f}_\infty \leq d+1$.
\end{lemma}

\begin{lemma}\label{lem:v-bound}
If $(\Phi, \cL)$ is globally observable, then for each pair of neighbours $a$ and $b$ there exists a function $f$ satisfying \cref{eq:est-def} such that
$\norm{f}_\infty \leq d^{1/2}(1 + k)^{d/2}$. If $(\Phi, \cL)$ is also locally observable, then $f$ can be chosen so that $f(c, \sigma) = 0$ for all $c \notin \cN_{ab}$. 
\end{lemma}

\begin{proof}
We prove only the first part. The proof for locally observable games is the same, but the signal matrices defined below are restricted to $c \in \cN_{ab}$.
Assume without loss of generality that $\Sigma = [d]$ and $d \geq 2$. For $c \in [k]$ define $S_c \in \{0,1\}^{d \times d}$ to be
the matrix with $(S_c)_{\sigma x} = 1$ if $\Phi(c, x) = \sigma$. Then let $S \in \{0,1\}^{d \times dk}$ be formed
by horizontally stacking the matrices $\{S_c : c \in [k]\}$.
By the definition of local observability it holds that $\ell_a - \ell_b \in \image(S)$.
Let $S^+$ be the Moore-Penrose pseudo-inverse of $S$ and let $w = S^+ (\ell_a - \ell_b)$, which satisfies $Sw = \ell_a - \ell_b$. 
Then $f$ can be chosen so that $\norm{f}_\infty = \norm{w}_\infty \leq \norm{w}_2$. 
Since losses are bounded in $[0,1]$ we have $\norm{w}_2 \leq \norm{S^+}_2 \norm{\ell_a - \ell_b}_2 \leq d^{1/2} \sigma_{\min}^{-1}$, where $\sigma_{\min}$ is the smallest nonzero singular value of $S$.
Hence we need to lower bound the smallest nonzero eigenvalue of $B = S S^\top$, which is a $d\times d$ matrix with entries in $\{0,1,\ldots,k\}$.
The characteristic polynomial of $B$ is $\chi(\lambda) = \det(\lambda I - B) = \sum_{i=0}^d a_i \lambda^d$, where
$a_d = 1$ and, up to a sign, $a_i$ is the sum principle minors of $B$ of size $d - i$.
Since the geometric mean is smaller than the arithmetic mean, for matrix $A \in [0,k]^{i \times i}$ it holds that $\det(A) \leq (\tr(A) / i)^i \leq k^i$.
Hence,
\begin{align*}
|a_{d-i}| \leq {d \choose i} k^i\,.
\end{align*}
By the binomial theorem,
\begin{align*}
\sum_{i=0}^d |a_i| \leq \sum_{i=0}^d {d \choose i} k^i = (1 + k)^d\,.
\end{align*}
Let $i_{\min} = \min\{i : a_i \neq 0\}$ and suppose that $\lambda > 0$ is the smallest nonzero root of $\chi$, which must be positive. Then
\begin{align*}
0
&= |\chi(\lambda)| 
= \left|\sum_{i=i_{\min}}^d a_i \lambda^i\right| = \lambda^{i_{\min}} \left| \sum_{i=i_{\min}}^d a_i \lambda^{i - i_{\min}}\right| 
\geq \lambda^{i_{\min}} \left|1 - (1+k)^d \lambda \right|\,,
\end{align*}
where we used the fact that $(a_i)$ are integer-valued.
Therefore $\lambda \geq (1+k)^{-d}$, which means that $\norm{S^+}_2 \leq (1+k)^{d/2}$ and hence $\norm{f}_\infty \leq d^{1/2} (1+k)^{d/2}$.
\end{proof}

\section{Figures and examples}\label{app:figures}

\paragraph{Finite partial monitoring example}
Below is a $4$-action finite-outcome, finite-action partial monitoring game with feedback set $\Sigma = \{\bot, \irain, \isnow, \isun\}$.
The left table shows the loss function and the right shows the signal function. By staying indoors you cannot evaluate the quality of the snow, but climbing or skiing in poor conditions is no fun.

\begin{figure}[H]
\scriptsize
\centering
\renewcommand{\arraystretch}{1.5}
\begin{tabular}{|l|lll|}
\hline
\textbf{Losses} $\bm \cL$      & \cellcolor{lightgray}\textsc{sun} & \cellcolor{lightgray}\textsc{snow} & \cellcolor{lightgray}\textsc{rain} \\ \hline

\cellcolor{lightgray}\textsc{ski}   & 3/4   & 0  & 1 \\
\cellcolor{lightgray}\textsc{climb} & 0   & 3/4  & 1 \\
\cellcolor{lightgray}\textsc{math}  & 1/2 & 1/2 & 1/4 \\
\cellcolor{lightgray}\textsc{raindance}  & 1 & 1 & 0 \\
\hline
\end{tabular}
\hspace{0.5cm}
\renewcommand{\arraystretch}{1.5}
\begin{tabular}{|l|ccc|}
\hline
\textbf{Signals} $\bm\Phi$      & \cellcolor{lightgray}\textsc{sun} & \cellcolor{lightgray}\textsc{snow} & \cellcolor{lightgray}\textsc{rain} \\ \hline
\cellcolor{lightgray}\textsc{ski}   & \sun   & \snow  & \rain \\
\cellcolor{lightgray}\textsc{climb} & \sun   & \snow  & \rain \\
\cellcolor{lightgray}\textsc{math}  & $\bot$ & $\bot$ & $\bot$ \\
\cellcolor{lightgray}\textsc{raindance} & \sun   & \snow  & \rain \\
\hline
\end{tabular}
\caption{Example finite partial monitoring game}\label{fig:example}
\end{figure}
The following figure shows the cell decomposition for the above game, $\Delta^{d-1}$ is parameterised by $(p, q, 1 - p - q)$.
In this game all actions a Pareto optimal. All actions are neighbours of \textsc{math} and otherwise \textsc{climb} and \textsc{ski} are neighbours
and \textsc{math} and \textsc{raindance}.
The game is locally observable because the loss of all actions can be identified by playing that action, except for \textsc{math}, the losses of which can be identified
by playing any of its neighbours.

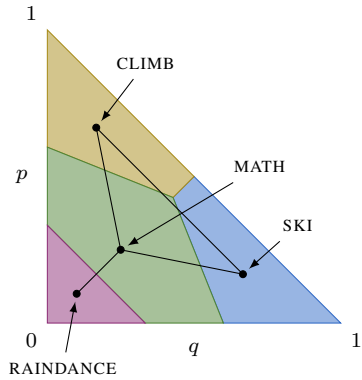
\begin{figure}[H]
\centering
\begin{tikzpicture}[font=\scriptsize,scale=1.3]
\tikzstyle{r1} = [fill=c1!50!white,draw=c1]
\tikzstyle{r2} = [fill=c2!50!white,draw=c2]
\tikzstyle{r3} = [fill=c3!50!white,draw=c3]
\tikzstyle{r4} = [fill=c4!50!white,draw=c4]
\draw[r1] (1.8,0) -- (1.285, 1.285) -- (1.5,1.5) -- (3,0) -- cycle;
\draw[r2] (0,1.8) -- (1.285, 1.285) -- (1.5,1.5) -- (0,3) -- cycle;
\draw[r3] (1,0) -- (1.8, 0) -- (1.285,1.285) -- (0,1.8) -- (0,1) -- cycle;
\draw[r4] (0,0) -- (1,0) -- (0,1) -- cycle;
\node[anchor=north east] at (0,0) {$0$};
\node[anchor=north west] at (3,0) {$1$};
\node[anchor=south east] at (0,3) {$1$};
\node[anchor=north] at (1.5,-0.1) {$q$};
\node[anchor=east] at (-0.1,1.5) {$p$};
\draw[fill] (2,0.5) circle (1pt);
\draw[fill] (0.5,2) circle (1pt);
\draw[fill] (0.75,0.75) circle (1pt);
\draw[fill] (0.3,0.3) circle (1pt);
\draw (2,0.5) -- (0.5,2) -- (0.75,0.75) -- cycle;
\draw (0.3,0.3) -- (0.75,0.75);
\node[anchor=west] (c) at (0.6, 2.66) {\textsc{climb}};
\node[anchor=west] (s) at (2.3, 1) {\textsc{ski}}; 
\node[anchor=west] (m) at (1.8,1.6) {\textsc{math}};
\node[anchor=north west] (r) at (-0.5,-0.3) {\textsc{raindance}};
\draw[thin,-latex,shorten >=3pt] (c) -- (0.5,2);
\draw[thin,-latex,shorten >=3pt] (s) -- (2, 0.5);
\draw[thin,-latex,shorten >=3pt] (m) -- (0.75, 0.75);
\draw[thin,-latex,shorten >=3pt] (r) -- (0.3,0.3);
\end{tikzpicture}
\caption{Cell decomposition for the game described above where $d = 3$. The figure shows $\Delta^{d-1}$ projected onto the plane by the parameterisation $(p, q, 1-p-q)$. All actions are Pareto optimal, so their cells
have dimension $d - 1 = 2$. The intersection of the cells of neighbouring actions are the lines shared by the cells, which have dimension $1$.}
\end{figure}

\paragraph{Tree construction}
The figure depicts the cell decomposition for a partial monitoring game with seven actions and the tree structure defined in Lemma~\ref{lem:tree}.
Arrows indicate the parent relationship. All paths leading towards $G_t$. Red nodes are descendants of $a$. 
Blue nodes are ancestors. Dotted lines indicate connections in the neighbourhood graph that are not part of the tree.

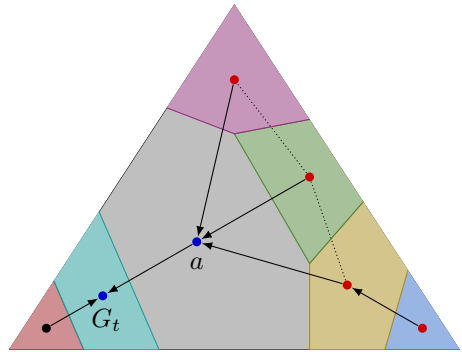
\begin{figure}[H]
\centering
\begin{tikzpicture}[font=\small,yscale=1.15]
\tikzstyle{p} = [fill=black,circle,inner sep=1.2pt]
\tikzstyle{a} = [fill=blue!80!black]
\tikzstyle{d} = [fill=red!80!black]
\tikzstyle{e} = [-latex]
\tikzstyle{g} = [densely dotted]

\clip (0,0) -- (6,0) -- (3,4) -- cycle;
\draw[fill=lightgray] (0,0) -- (6,0) -- (3,4) -- cycle;
\draw[fill=c1!50!white,draw=c1] (5,0) -- (5.5,2) -- (6,0) -- cycle;
\draw[fill=c2!50!white,draw=c2] (4,0) -- (4,1) -- (6,3) -- (5,0) -- cycle;
\draw[fill=c3!50!white,draw=c3] (4,1) -- (3,2.5) -- (6,3) -- cycle;
\draw[fill=c4!50!white,draw=c4] (3,2.5) -- (0,3.5) -- (3,4) -- (6,3) -- cycle;
\draw[fill=c5!50!white,draw=c5] (0,0) -- (1,0) -- (0,2) -- cycle;
\draw[fill=c6!50!white,draw=c6] (1,0) -- (2,0) -- (0,4) -- (0,2) -- cycle;

\node[p] (6) at (0.5,0.25) {};
\node[p,a] (5) at (1.25,0.625) {};
\node[p,a] (4) at (2.5,1.25) {};
\node[p,d] (3) at (3,3.125) {};
\node[p,d] (2) at (4,2) {};
\node[p,d] (1) at (4.5,0.75) {};
\node[p,d] (0) at (5.5,0.25) {};

\draw[e] (6) -- (5);
\draw[e] (4) -- (5);
\draw[e] (3) -- (4);
\draw[e] (2) -- (4);
\draw[g] (2) -- (3);
\draw[e] (1) -- (4);
\draw[e] (0) -- (1);
\draw[g] (1) -- (2);

\node at (1.3,0.35) {$G_t$};
\node at (2.5,1) {$a$};

\end{tikzpicture}
\caption{Tree construction}\label{fig:tree}
\end{figure}
\section{Technical calculation}\label{app:lem:fundamental}

\begin{lemma}\label{lem:fundamental}
Let $P_{ta} = \bbP_t(A_t = a)$. Then the following hold almost surely:
\begin{align*}
\E_t[D_F(P_{t+1}^*, P_t^*)] 
&= \sum_{a=1}^k P_{ta} \E_t\left[D_F(\bbP_{t,A^*|\Phi_t(a)}, \bbP_{t,A^*})\right] \,, \\
I_t(A^* ; \Phi_t(A_t), A_t) 
&= \sum_{a=1}^k P_{ta}^* \sum_{b=1}^k P_{tb} \,\E_t\left[\bKL{\bbP_{t,\Phi_t(b) | A^* = a}}{\bbP_{t,\Phi_t(b)}}\right] \,.
\end{align*}
\end{lemma}

\begin{proof}
Recall that $P_{t+1}^* = \bbP_{t+1}(A^* \in \cdot) = \bbP_t(A^* \in \cdot \mid A_t, \Phi_t(A_t))$.
Then
\begin{align*}
\E_t[D_F(\bbP_{t,A^*|A_t,\Phi_t(A_t)}, \bbP_{t,A^*})]
&= \E_t\left[\E_t[D_F(\bbP_{t,A^*|A_t,\Phi_t(A_t)}, \bbP_{t,A^*}) \mid A_t]\right] \\
&= \E_t\left[\E_t[D_F(\bbP_{t,A^*|\Phi_t(A_t)}, \bbP_{t,A^*}) \mid A_t]\right] \\
&= \sum_{a=1}^k P_{ta} \E_t[D_F(\bbP_{t,A^*|\Phi_t(a)}, \bbP_{t,A^*}) \mid A_t = a] \\
&= \sum_{a=1}^k P_{ta} \E_t[D_F(\bbP_{t,A^*|\Phi_t(a)}, \bbP_{t,A^*})]\,,
\end{align*}
where in the second and fourth inequalities we used the independence of $A_t$ and $X$ under $\bbP_t$.
The second part of the lemma follows from an identical argument.
\end{proof}

\end{document}